\documentclass[journal,twoside,10pt,compsoc]{IEEEtran}
\def \arxiv{0}
\usepackage {setspace}
\usepackage {color}

\usepackage{cite}

\ifCLASSINFOpdf
  \usepackage[pdftex]{graphicx}
\else
   \usepackage[dvips]{graphicx}
   \DeclareGraphicsExtensions{-eps-converted-to.pdf}
\fi
\usepackage{bm}

\usepackage{amsmath}
\usepackage{amsfonts}
\usepackage{amssymb}
\usepackage{amsthm}
\newtheorem{definition}{Definition} 
\newtheorem{theorem}{Theorem}
\newtheorem{lemma}{Lemma}
\theoremstyle{remark}
\newtheorem*{remark}{Remark}
\usepackage{stfloats}
\usepackage{bm}

\usepackage{algorithmic}
\usepackage{algorithm}

\usepackage{amsmath}

\usepackage{array}

\ifCLASSOPTIONcompsoc
  \usepackage[caption=false,font=normalsize,labelfont=sf,textfont=sf]{subfig}
\else
  \usepackage[caption=false,font=footnotesize]{subfig}
\fi
\usepackage{bbm}
\begin{document}
%
\title{An Efficiency-boosting Client Selection Scheme for Federated Learning with Fairness Guarantee }
%
%
%

 \author{Tiansheng~Huang,~
          Weiwei~Lin,~
          Wentai~Wu,~
          Ligang~He,~
          Keqin~Li,~\IEEEmembership{Fellow,~IEEE},
          and ~Albert Y. Zomaya,~ \IEEEmembership{Fellow,~IEEE}
 \IEEEcompsocitemizethanks{ \IEEEcompsocthanksitem T. Huang and W. Lin (corresponding author) are with the School of Computer Science and Engineering, South China University of Technology, China. Email: cs\_tianshenghuang@mail.scut.edu.cn, linww@scut.edu.cn.}
 \IEEEcompsocitemizethanks{\IEEEcompsocthanksitem  W. Wu, L. He are with the Department of Computer Science, the University of Warwick. Email: wentai.wu, Ligang.He@warwick.ac.uk.}
\IEEEcompsocitemizethanks{\IEEEcompsocthanksitem K. Li is with the Department of Computer Science, State University of New York, New Paltz, NY 12561 USA. E-mail: lik@newpaltz.edu.
}
\IEEEcompsocitemizethanks{\IEEEcompsocthanksitem A. Y. Zomaya is with the School of Computer Science, The University of Sydney, Sydney, Australia. Email: albert.zomaya@sydney.edu.au.
}
 }

%
%

\markboth{THIS PAPER HAS BEEN ACCEPTED BY THE IEEE TRANSACTIONS ON PARALLEL AND DISTRIBUTED SYSTEMS (TPDS).  DOI 10.1109/TPDS.2020.3040887}%
{Huang \MakeLowercase{\textit{et al.}}: An Efficiency-boosting Client Selection Scheme for Federated Learning with Fairness Guarantee}
%


\IEEEtitleabstractindextext{%
\begin{abstract}
The issue of potential privacy leakage during centralized AI's model training has drawn intensive concern from the public. A Parallel and Distributed Computing (or PDC) scheme, termed Federated Learning (FL), has emerged as a new paradigm to cope with the privacy issue by allowing clients to perform model training locally, without the necessity to upload their personal sensitive data. In FL, the number of clients could be sufficiently large, but the bandwidth available for model distribution and re-upload is quite limited, making it sensible to only involve part of the volunteers to participate in the training process. The client selection policy is critical to an FL process in terms of training efficiency, the final model's quality as well as fairness. In this paper, we will model the fairness guaranteed client selection as a Lyapunov optimization problem and then a $\rm \mathbf {C^2MAB}$-based method is proposed for estimation of the model exchange time between each client and the server, based on which we design a fairness guaranteed algorithm termed RBCS-F for problem-solving. The regret of RBCS-F is strictly bounded by a finite constant, justifying its theoretical feasibility. Barring the theoretical results, more empirical data can be derived from our real training experiments on public datasets.
\end{abstract}
}

\maketitle


%
\IEEEpeerreviewmaketitle
\section{Introduction}
\subsection{Background}
\IEEEPARstart{F}{ederated} Learning (FL) has been esteemed as one of the most promising solutions to the crisis known as isolated "data island". It helps break down the obstacles between parties or entities, allowing a greater extent of data sharing. All the entities being involved could benefit from such a new paradigm, in which model owners could build a more robust and comprehensive model with more data being accessible. Meanwhile, data owners might either receive substantial rewards or services that match their interests in return. More importantly, the privacy of the data owners would not risk being intruded since their raw data simply does not necessarily need to leave the local devices, as all the training is only performed locally.
\subsection{Motivations}
Within such a novel paradigm, new challenges co-exist with opportunities. Unlike the traditional model training process, not all the data within the system could be accessed over every round of training. Owing to the limited bandwidth and the dynamic status of the training clients, only a fraction of them could be picked to perform training on behalf of the model owner. From the perspective of a model owner, the selection decision in each round could have a profound impact on the model's training time, convergence speed, training stability, as well as the final achieved accuracy. Some studies in the literature have made iconic contributions to this problem. To illustrate, in \cite{nishio2019client}, when making a selection, Nishio et al. concentrate on the evaluation of communication time, which accounts for a considerable portion of time for a training round. In another study\cite{zeng_energy-efficient_2019}, the authors consider more. They further take the energy consumption factor into consideration. Barring an intelligent decision on participant selection, an efficient bandwidth allocation scheme was also given by them. However, the current line of research evades two important factors. For one thing, both of them assume a pre-known local training time to the scheduler, which may not be realistic in all circumstances. For another, indicated by Theorem 2 in \cite{zeng_energy-efficient_2019}, devices with higher performance are more favored by their proposed methods. Indeed, always selecting the “fast” devices somehow boost the training process. But clients with low priority are simply being deprived of chances to participate at the same time, which we refer to it as an unfair selection among clients. In fact, such an extreme selection scheme might bring undesirable side effects by neutralizing some portions of data. Conceivably, with a smaller amount of data involved, data diversity can not be guaranteed, thereby hurting the performance of model training to some extent. This motivates us to develop an algorithm that strikes a good balance between training efficiency and fairness. Also, the algorithm is supposed to be intelligent enough to predict the training time of the clients based on their reputation (or their historical performance), rather than assuming it to be known a priori.
\subsection{Contributions}
The main contributions of this paper are listed as follows:
{
\begin{enumerate}
\item We investigate the client selection in FL from the perspective of minimizing average model exchange time when subjecting to a relatively flexible long-term fairness guarantee, as well as a few rigid system constraints. At the same time, more factors, involving the clients' availability, unknown and stochastic training time, as well as the dynamic communication status, are taken into account.
\item Inspired by \cite{li2019combinatorial}, we transform the original offline problem into an online Lyapunov optimization problem where the long-term guarantee of client participating rate is quantified using dynamic queues.
\item We build a Contextual Combinatorial Multi Arm Bandit ($\rm C^2MAB$) model for estimation of the model exchange time of each client based on their contextual properties and historical performance (or their reputation).
\item A fairness guaranteed selection algorithm RBCS-F is proposed for efficiently resolving the proposed optimization problem in FL. Theoretical evaluation and real data-based experiments show that RBCS-F can ensure no violation in the long-term fairness constraint. Besides, the training efficiency has been significantly enhanced, while the final model accuracy remains close, in a comparison with random, i.e., the vanilla client selection scheme of FL.
\end{enumerate}
}
To the best knowledge of the authors, this is the first trackable practice that combines Lyapunov optimization and $\rm C^2MAB$ for a long-term constrained online scheduling problem. Also, we shall remind the readers that the proposed combination does not confine to the application of our current proposed problem, but it has the potential to extend to a wider range of selection problems. (e.g. worker selection in crowdsensing, channel selection in the wireless network, etc.)
\par
\section{Related Works}
In recent years, we are experiencing a great surge of Edge Intelligence (see in \cite{deng_edge_2019,wang_-edge_2018,zhou_edge_2019}). Numerous attempts have been made to combine AI techniques and edge, tapping the profound potential of the ubiquitous deployed edge devices. Among these, one of the most iconic studies could be neurosurgeon \cite{kang2017neurosurgeon}. Its basic idea is to partition an intact DNN (Deep Neural Networks) into several smaller parts and disseminate them to the edge devices. Owing to a low latency between edge and users, inference speed could be significantly improved.
\par Besides, edge coordinated Federated Learning is another promising combination. Federated Learning\cite{mcmahan2016communication}, which allows data to be trained in local rather than being transmitted to the cloud, is now known as a more secure paradigm for AI's model training. We have witnessed the surge of some plausible applications of FL within these years (e.g. keyboard and emoji prediction in \cite{hard2018federated,ramaswamy2019federated}, visual object detection in \cite{liu2020fedvision}, etc). Despite the potential advantages as well as the promising applications of FL, the communication overhead between cloud and users renders as a bottleneck for it. A lengthy communication round during training might significantly degrade FL's training performance. Although more advance training schemes, such as federated distillation (FD, originally proposed in \cite{jeong2018communication}), promise us a more desirable, reduced size information exchange between users and model aggregator, the latency between cloud and edge alone is inevitable. Such a defect could be better addressed by making edge the model's aggregator or at least an intermediate one (see in \cite{liu2019edge}). In this way, the data don't have to bear an outstanding communication length to the cloud. Another open problem of FL we would like to mention here is the client selection problem, originally proposed in \cite{nishio2019client} and followed by some related works (e.g. \cite{zeng_energy-efficient_2019,yoshida2019hybrid,ye2020federated,yang2019scheduling}). Many of them see the problem from a communication perspective, focusing on building an efficient selection or bandwidth allocation scheme that helps shorten the communication length. In this paper, we will see the problem from a different angle, namely, to investigate how the fairness factor affects the training performance. We couldn't check out any specialized studies on this topic yet and we hope our research could bring some new insights in the field. Last but not least, we also want to note, FL itself is now far from its maturity, many important issues worth our study. Some of which might involve asynchronous or semi-asynchronous aggregation protocol \cite{xie2019asynchronous, wu2019safa}, incentive mechanism \cite{kang2019incentive, khan2019federated} and security issues\cite{bagdasaryan2018backdoor}, etc. We look forward to more insightful and dedicated research into FL.
\par
Now we would like to talk more about a classical problem, termed multi-arm bandit (MAB). In a classical MAB setting, arms are characterized by different unknown reward distribution. In each round of play, the player selects one of the arms from the possible options and gains a reward sampling from the selected arm's reward distribution. As there exists a tradeoff between exploration and exploitation for the player, how to maximize her obtained reward is the main concern. Several solutions, such as the well-known Upper Confidence Bound (UCB), Thompson Sampling (TS) could be applied to the problems. In addition, MAB has several variations. Those include combinatorial MAB, where players are allowed to select more than one arms in every round, contextual MAB \cite{li2010contextual,abbasi2011improved}, where the reward of an arm follows a linear stochastic formulation, and a much newer one, contextual combinatorial MAB ($\rm C^2MAB$) \cite{qin2014contextual,li2016contextual}, which is the combination of the above two. We found that $\rm C^2MAB$ could be well applied to the client selection problem in FL, as each client could be regarded as an arm and our task for each round is to choose a combination of which for participation, thus, in this paper, such a prototype will be used for our model establishment.
\section{Preliminary Introduction on FL}
\begin{figure}[!hbtp]
\centering
\includegraphics[width=3.5in]{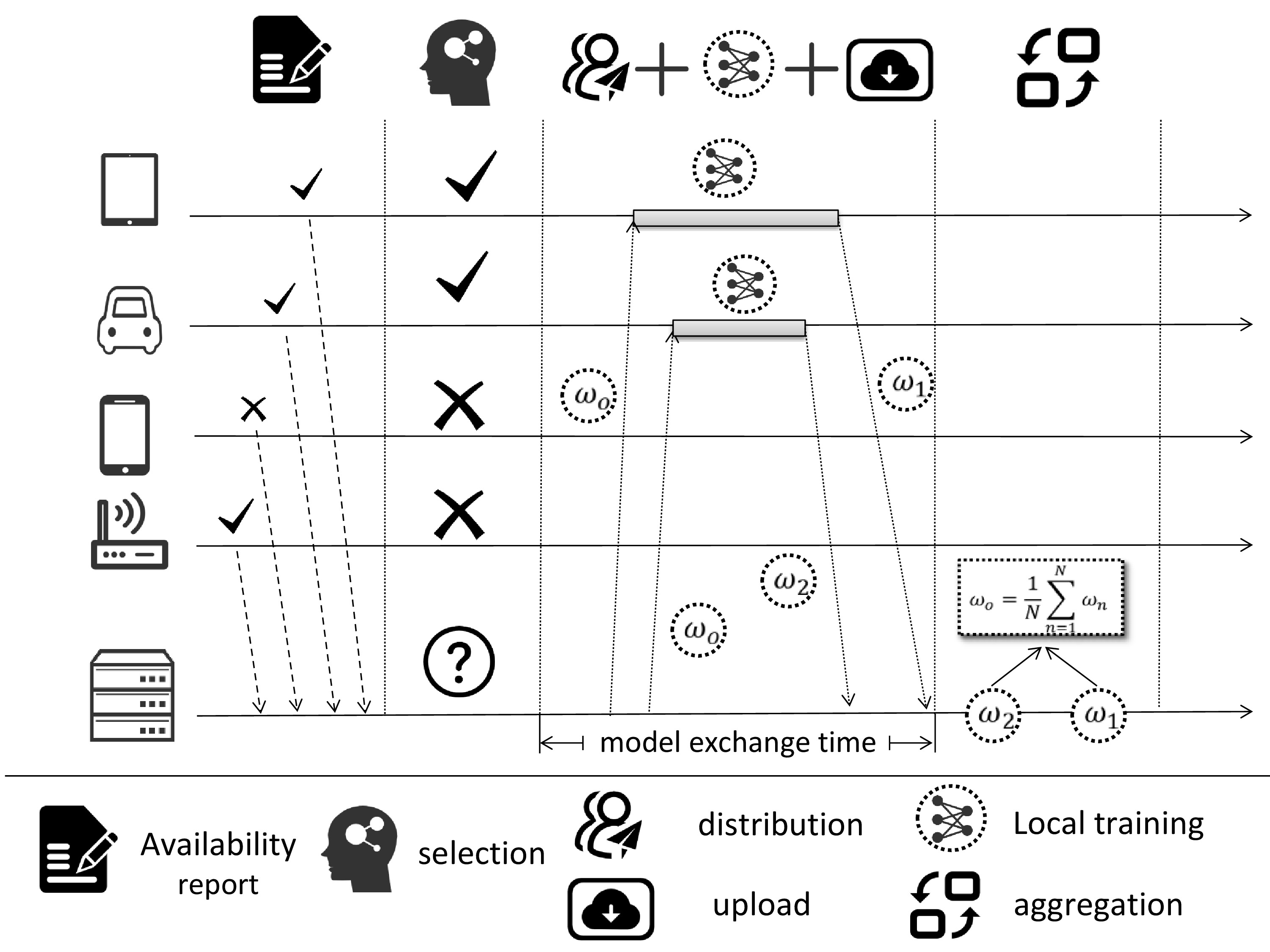}
\caption{Illustration of FL}
\label{system FL}
\end{figure}
In this paper, we consider an edge-coordinated federated learning system, in which edge is functioning as a model aggregator, and the clients (mostly mobile devices) are responsible for doing local training over their private data on behalf of the model's owner. We adopt in our system the most-accepted synchronous scheme for federated learning, which is characterized by training in iterations. For clearness, now we will explicitly explain the workflow of our synchronized scheme by giving four sequential stages of training, as follows:
\begin{enumerate}
\item At the very beginning of a new iteration, the clients first report their willingness to participate in the training as well as a few client-side information, which will be used for the client selection in the next stage.
\item In the second step, the scheduler conducts client selection to choose a portion of participants among the volunteers in light of the provided information.
\item Global model is distributed to the selected clients. After receiving the model, the clients conduct local training using their private data and update their local model. Once the training of all the selected clients is finished, the local model will be returned to the MEC server. The time span of this round is known as \textit{model exchange time}.
\item The collected local models are aggregated by the server, substituting the original global model that once being distributed, and then it proceeds to step 1) to start a new iteration.
\end{enumerate}
For a more vivid presentation of the training process, we refer the readers to Fig. \ref{system FL}.
\section{Problem Formulation}
Our main concern focuses on the selection phase, in which the server makes a decision on the involved clients. Before our formal introduction of the selection problem, we first derive a high-level description of the content of this section. In the first sub-section, we formulate the client selection problem into an offline problem with a long-term fairness constraint. The formulated problem is simple in form but indeed unsolvable due to the time coupling effect as well as the unknown model exchange time persisting in the objective. To resolve the time coupling effect, we transform the problem into an online mode using Lyapunov optimization technique, the online transformation of which gains us a fighting chance to derive an estimated model exchange time before each round scheduling, which might help resolve another obstacle (i.e. the unknown parameter in the objective function). Specifically, targeting the transformed online problem, a $\rm {C^2MAB}$ setting could come in handy for online learning of the exchange time, and being enlightened by which, we are able to further transform the problem into the ultimate form, which concludes the whole section. \par
Then we need to explain some key notations that are consistently used throughout the paper, among which, a set $ \mathcal{T} \triangleq \{1,2,\dots\}$, indexed by $t$, is used to capture the \textit{federated rounds} (namely, the iterations in FL's model update process). The set $\mathcal{N}\triangleq \{1,2,\dots N\}$ captures all the clients (each indexed by $n$) in the system. Besides, we assume that the maximum number of selected clients each round is fixed in advance to $m$. Another important notation is $\mathcal{S}_t$, which we use to capture the selected clients in round $t$ and it serves as the representation of the selection policy that we aim to optimize.  \par

\subsection{Basic Assumption on System Model}
\subsubsection{Model Exchange Time}
In a client selection problem, an important metric we shall evaluate is the long-term average model exchange time. We refer to the model exchange time as the time span between the instant the scheduler made the selection decision and that when all the re-upload models have been gathered. This model exchange time might involve time for model distribution, model training and model upload. Intuitively, a client selection scheme that is able to achieve a shorter span of each federated round is of interest, since a shorter period of each round explicitly marks shorter time for fix rounds of training. Recall that the server could step into the next phase (model aggregation) only after all the models have been gathered when adopting a synchronous federated training protocol. The time for model exchange is explicitly determined by the participated clients, or more precisely, by the one among them who spends the most time in training and model uploading. Mathematically, we have the following equation to capture the time span for a federated round:
\begin{equation}
\label{target}
f(\mathcal{S}_t,{\boldsymbol{\tau}_t})= \max_{n \in \mathcal{S}_t}\{ \tau_{t,n}  \}  
\end{equation}
where we use a set $\mathcal{S}_t$ to capture the selected clients in round $t$. Besides, $\tau_{t,n}$ is used to represent the time span between the very beginning of model distribution and the instant when the model from client $n$ being gathered. Here $\boldsymbol{\tau}_t\triangleq\{ \tau_{t,n}\}_{n\in\mathcal{N}}$ in round $t$ is unknown to the scheduler until the end of this round.
\subsubsection{Long-Term Fairness Constraint}
Another metric that might have a significant impact on FL's performance is fairness. Assume an ideal case that the server is fully aware of the exact model exchange time of each client for the incoming federated round. Then is it incontrovertibly optimized when always choosing the $m$-fastest clients, making the time span for each round of training minimized? We must note, however, that the answer may not be such apparent. We acknowledge that the time span of each round could be somehow minimized by adopting such a greedy selection scheme, but we must argue that if we always choose the fastest clients, small chance could become available for their slower counterparts, implicitly implying that little contribution could be obtained from the slowers' local data. Very likely, along with the selection bias, the global model would suffer a degradation on its capability to generalize. In this regard, a greedy selection may not trivially be the best scheme, and fairness in selection is another factor that we need to take into account. To model such a critical fairness concern, we introduce a long-term fairness constraint, as follows:
\begin{equation}
\label{fairness}
 \lim _{T \rightarrow \infty} \frac{1}{T} \sum_{t=1}^{T} \mathbb{E}[x_{t,n}] \geq \beta \quad \forall n \in \mathcal{N}
 \end{equation}
where $\beta$ models the expected guaranteed chosen rate of clients. $x_{t,n}$ is used to indicate whether client $n$ is involved in the federated round $t$ or not. In other words, $x_{t,n}=1$ for $n \in \mathcal{S}_t$; otherwise, $x_{t,n}=0$. The constraint is set to make sure the long-term average chosen rate of every client at least greater than $\beta$, which somehow helps maintain some degrees of fairness for the system.
\subsubsection{Availability of Clients }
As we are investigating a client selection problem under a highly dynamic real-world system, it is unrealistic to assume clients are always ready to provide training services. In fact, clients are free to join and leave the loose “federation” at any time they want. With this consideration, we use an indicator function $I_{t,n}$ to capture the status of a client, indicating whether the client is willing to engage or not. Such information could be given by the availability report from the clients before scheduling. Formally, we introduce a strict constraint to prevent futile participation:
\begin{equation}
\label{availability}
I_{t,n}  = 1  \quad \forall n \in \mathcal{S}_t
 \end{equation}
 \subsubsection{Selection Fraction}
Recall that the maximum number of clients that could be selected is fixed to $m$ in our setting. However, as the number of volunteers may not be able to reach $m$ if the activated number is smaller than $m$, we have to use a “$\min$” function to constraint the selection fraction, as follows:
 \begin{equation}
 \label{pick number}
|\mathcal{S}_t | =\min \left\{ m, \sum_{n\in \mathcal{N}} I_{t,n} \right\}
 \end{equation}
where $|\mathcal{S}_t |$ means the number of elements in $\mathcal{S}_t $.
Intuitively, in the case when the total number of availability could not overtake the maximum selection fraction, we simply involve all the active clients for the incoming round of training.
\subsection{An Offline Long-Term Optimization Problem}
Based on the above discussion, we are ready to introduce our client selection problem, as follows:
\begin{equation}
\begin{aligned}\textit{(P1)}:  \min_{\{\mathcal{S}_1,\mathcal{S}_2, \dots, \mathcal{S}_\infty\}} \quad & \lim _{T \rightarrow \infty} \frac{1}{T} \sum_{t=1}^{T}  \mathbb{E}[f(\mathcal{S}_t,\boldsymbol{\tau}_t)] \\ \text { s.t. } & ( \ref{fairness}), ( \ref{availability}),( \ref{pick number}) \end{aligned}
\end{equation}
where $\mathcal{S}_t$ captures the selected clients in each round, which is our optimized target. Intuitively, our aim is to minimize the long-term model exchange time while subjecting to a “soft” long-term fairness constraint (\ref{fairness}), which tolerates short-term violation, as well as two extra “hard” constraints (\ref{availability}), (\ref{pick number}), which bear no compromise.
\par
One could notice that \textit{P1} is a time-coupling scheduling problem, regarding the long-term objective and the fairness constraint in (\ref{fairness}). But we note here that such an optimization problem is challenging or even impossible to be solved offline. There are mainly three concerns about this. Firstly, random events, such as clients' availability, are not known to the scheduler until the very beginning of a particular round. This implies that an offline strategy, which is not given access to this particular information, can hardly guarantee the qualifications of constraints (\ref{availability}) and (\ref{pick number}). Our second concern is derived from the time-coupling constraint (\ref{fairness}), which is quite difficult for the offline solution to deal with. The final concern is that the information on model exchange time can only be observed after actually involving the clients in training. Nevertheless, the scheduler is supposed to make a scheduling decision before the real training process, when the actual model exchange time is unachievable. The lack of this crucial information precludes any feasible attempts to achieve an optimal offline solution. Therefore, for an alternative sub-optimal problem-solving, in the following section, we will elaborate on our transformation of the offline problem to a step-by-step online scheduling problem by Lyapunov optimization to cope with the first two proposed concern. Later, we will display our estimation of model exchange time based on clients' reputation, by which we leverage to deal with our third concern.
\subsection{Problem Transformation under Lyapunov Framework }
In this sub-section, we first take advantage of Lyapunov optimization framework to transform the offline problem \textit{P1} to an online one. \par
First, we introduce a virtual queue for each client, whose backlog$\footnote{We use the term “backlogs” and “queue length” interchangeably throughout the paper but actually, they share the same meaning.}$ is denoted by $Z_{t,n}$$\footnote{The subscripts $t$ and $n$ here correspond to a federated round and a client, respectively. A similar form of subscript definition will be adopted throughout the paper. }$, to transform the long-term fairness constraint. Specifically, $Z_{t,n}$ evolves across the FL process complying the following rule:
 \begin{equation}
 \label{Z queue}
Z_{t+1,n}=\left[Z_{t,n}+\beta- x_{t,n}\right]^+
\end{equation}
where $\beta$ is the expected guaranteed selection rate in (\ref{fairness}) and $[\dots]^+$ is equivalent to $\max(\dots,0)$. 
\par
Now we present Theorem \ref{fairness queue theorem} to justify the rationale for this transformation.
\begin{theorem}
\label{fairness queue theorem}
Long-term time average constraint (\ref{fairness}) holds  if all the virtual queues (whose backlogs denoted by $Z_{t,n}$) remain  \textit{mean rate stable} across the FL process.
\end{theorem}
\begin{proof}
According to the queue theory (see in Theorem 2.5, \cite{neely2010stochastic}), if all the virtual queues  $Z_{t,n}$ remain \textit{ mean rate stable} across the FL process (or formally, $\lim _{T \rightarrow \infty} \mathbb{E}[Z_{T,n}]/T=0$), the time average arrival rates of the queue will be smaller than the service rates, namely, we have:
\begin{equation}
\frac{1}{T}\lim _{T \rightarrow \infty} \sum_{t=1}^{T}\mathbb{E}[\beta- x_{t,n} ] \leq 0
\end{equation}
Through basic mathematics operations, we can reconstruct the above inequality into the form of (\ref{fairness}) with ease. This completes the proof.
\end{proof}
\par
\begin{remark}
Intuitively, the length of the queue will soar towards infinity if the long-term fairness constraint is violated, (i.e. when the real chosen rate could not match up with the expected guaranteed selection rate), which is formally justified by Theorem \ref{fairness queue theorem}. To guarantee the fairness constraint, the queue has to remain \textit{mean rate stable} and a qualified algorithm is supposed to achieve this goal. Apart from this conclusion, we shall note that the stabilized queue length could also reflect the degree of fairness. For example, if a client never being selected in the first few limited round, its corresponding queue length will soar to a positive value. After that, if its real selection rate basically flats with the expected guaranteed selection rate, its queue still remains \textit{mean rate stable} and the queue length will slightly fluctuate over the same positive value. Intuitively, the bigger this value is, the unfairer the selection policy could be, as it demonstrates more violation of the fairness constraint in the initial stage. This conclusion could also be derived from the results in our experiments, which will be presented later.
\end{remark}
With Theorem \ref{fairness queue theorem}, now we have transformed the troublesome time-coupling constraint into the goal of ensuring the virtual queues mean rate stable across the FL process.  To reach this end, a straightforward approach is to bound every increase of queues so that they could not grow to infinity. Under this motivation, we shall leverage Lyapunov optimization technique to bound the growth of virtual queues while simultaneously minimizing the objective in \textit{P1}. First, we establish the quadratic Lyapunov function, with the following form:
\begin{equation}
\label{lyapunov function}
 \mathcal {L}(\boldsymbol{\Theta}(t))=\frac{1}{2}\sum_{n\in \mathcal{N}} Z_{t,n}^2
\end{equation}
where $\boldsymbol{\Theta}(t)\triangleq \{Z_{t,n}\}_{n\in\mathcal{N}}$ contains the backlogs of all the virtual queues.
\par
Aiming at bounding the expected increase of $\mathcal {L}(\boldsymbol{\Theta}(t))$ for one single round, we first formulate the \textit{Lyapunov drift} to measure it, basically, we have:
\begin{equation}
\label{drift}
 \Delta(\boldsymbol{\Theta}(t))=\mathbb{E}[ \mathcal {L}(\boldsymbol{\Theta}(t+1))-\mathcal {L}(\boldsymbol{\Theta}(t))  | \boldsymbol{ \Theta}(t)   ]
\end{equation}
As the backlogs of queues $\boldsymbol{\Theta}(t)$ can be known to the scheduler when being scheduled in an online manner, we take it as the condition in the Lyapunov drift. It is notable that the conditional expectation here is with respect to the availability of clients (which is a stochastic variable) as well as the possibly random selection policy. For ease of later interpretation, we let $\omega_t \triangleq \{I_{t,n}\}_{n\in \mathcal{N}}$ to capture the stochastic availability.\par
Recall that the objective of \textit{P1} is to minimize the model exchange time while satisfying the given constraints. This motivates us to combine the objective function into the drift function. Formally, we term such a combination as \textit{drift-plus-cost} function, with the following form:
  \begin{equation}
  \Delta(\boldsymbol{\Theta}(t)) + V \mathbb{E}[ f(\mathcal{S}_t,\boldsymbol{\tau}_t)  |  \boldsymbol{\Theta}(t) ]
\end{equation}
where $V \geq 0$ is a penalty factor set for the purpose of balancing the tradeoff between minimizing the objective and satisfying the fairness constraint. Such a parameter is crucial for the algorithm's performance and we will conduct a specific analysis to it in the next section. Note that the conditioned expectation being taken here is also with respect to stochastic events $\omega(t)$ and the possibly random policy as well. Now we are going to introduce a potential upper bound for the drift-plus-cost function. We show the result by Theorem \ref{first upper bound for drift plus cost}.
\begin{theorem}
\label{first upper bound for drift plus cost}
Conditioning on the queues' backlogs $\boldsymbol{\Theta}(t)$, the drift-plus-cost function for our system model could be bounded into the following form, where $\Gamma=N\left(1+\beta^2 \right)/2$ is a constant.
  \begin{align}
  \label{drift plus cost bound}
\begin{split}
  &\Delta(\boldsymbol{\Theta}(t))+V \mathbb{E}[f(\mathcal{S}_t,\boldsymbol{\tau}_t)|\boldsymbol{\Theta(t)} ] \\ &\leq  \Gamma+ \sum_{n \in \mathcal{N}} Z_{t,n} \mathbb{E}[ \beta-x_{t,n} |  \boldsymbol{\Theta}(t)  ]+V \mathbb{E}[ f(\mathcal{S}_t,\boldsymbol{\tau}_t)| \boldsymbol{\Theta}(t)) ]
\end{split}
\end{align}
\end{theorem}
\if\arxiv1
\begin{proof}
Taking square of  (\ref{Z queue}), we have 
 \begin{align}
\begin{split}
 Z^2_{t+1,n} = Z^2_{t,n}+2Z_{t,n}\left(\beta-x_{t,n} \right)+  \left(\beta-x_{t,n}\right)^2
\end{split}
\end{align}
Then the difference between $\frac{1}{2}Z^2_{t+1,n}$ and $\frac{1}{2}Z^2_{t,n}$ becomes:
\begin{align}
\label{Z difference}
 \begin{split}
&\frac{1}{2}\left(Z_{t+1,n}^2-Z_{t,n}^2\right)\\&= \frac{1}{2}\left(\beta-x_{t,n}\right)^2+Z_{t,n}(\beta-x_{t,n} )  \\ 
&\leq \frac{1}{2}\left( x_{t,n}^2+ \beta^2 \right)+Z_{t,n}\left(\beta-x_{t,n}\right)\\
&\overset{(a)}\leq \frac{1}{2} ( 1+\beta^2 )+Z_{t,n}\left(\beta-x_{t,n}\right)
\end{split}
\end{align}
Among which, (a) is valid since $x_{t,n}\in\{0,1\}$, trivially we have $x_{t,n}^2\leq 1$. 

Now combining (\ref{lyapunov function}), (\ref{drift}) and (\ref{Z difference}), it yields: 
\begin{align}
\label{drift bound}
 \begin{split}
\Delta(\Theta(t)) \leq & \Gamma+ \sum_{n\in\mathcal{N}} Z_{t,n} \mathbb{E}[  \beta-x_{t,n} |  \boldsymbol{\Theta}(t)  ]\\
\end{split}
\end{align}
where  $\Gamma= N\left(1+\beta^2 \right)/2 $.
\par
Plugging $V \mathbb{E}[  f(\mathcal{S}_t,\boldsymbol{\tau}_t) | \boldsymbol{\Theta}(t)) ] $ into (\ref{drift bound}), it can smoothly transform to the form in (\ref{drift plus cost bound}). This completes the proof.
\end{proof}
\else
\begin{proof}
The complete proof is given in Appendix \ref{Proof1}.
\end{proof}
\fi 
Intuitively, if we minimize the Right Hand Side (R.H.S) of  (\ref{drift plus cost bound}), the fairness virtual queues could be somehow maintained stable, while the objective function is also being minimized. Now shall introduce our step-by-step online scheduling problem by giving \textit{P2}:
  \begin{align}
\begin{split}
 \textit{(P2)}:  &\min_{ \bm x_{t}   }\quad \Gamma+\sum_{n\in \mathcal{N}} Z_{t,n} (\beta-x_{t,n})+V  \dot{f}(\bm x_t,\boldsymbol{\tau}_t)\\
s.t. \quad&  \sum_{n\in \mathcal{N}} x_{t,n}=\min \left\{ m, \sum_{n\in \mathcal{N}} I_{t,n} \right\}\\
& x_{t,n} \leq I_{t,n}\\
& x_{t,n}  \in \{0,1\}
\end{split}
\end{align}
we first have to make it clear that we use $\bm x_{t}$ to substitute all the $\mathcal{S}_t$ in \textit{P1}, making it a clearer form. Here $\dot{f}(\bm x_t,\boldsymbol{\tau}_t)=\max_{n \in \mathcal{N}}\{ x_{t,n} \tau_{t,n} \}$ is an equivalent form to $f(\mathcal{S}_t,\boldsymbol{\tau}_t)$.
While solving \textit{P2} on every round, the R.H.S of (\ref{drift plus cost bound}) can be minimized. The rationale behind is quite evident. As we have done the minimization under every round (alternatively, under every $\omega_t$, since $\omega_t$ is an independent sampling for each round), then the expectation with respect to $\omega_t$ is also being minimized. Note here that $\omega_t$ is indeed observable for an online algorithm since an online algorithm makes scheduling after the stage of availability report, making it accessible to this particular information.
\par
For briefness, we eliminate all the constants (i.e. $\Gamma$, $Z_{t,n}\beta$) in the objective of \textit{P2} and transform it to \textit{P3}:
  \begin{align}
\begin{split}
\textit{(P3)}:&  \min_{ \mathbf{x}_t} \quad   V \max_{n \in \mathcal{N}}\{ x_{t,n} \tau_{t,n}  \}-\sum_{n\in \mathcal{N}}Z_{t,n} x_{t,n}\\
s.t. \quad&  \sum_{n\in \mathcal{N}} x_{t,n}=\min \left\{ m, \sum_{n\in \mathcal{N}} I_{t,n} \right \}\\
& x_{t,n} \leq I_{t,n}\\
& x_{t,n}  \in \{0,1\}
\end{split}
\end{align}

But note that such a problem remains unsolvable yet since the real model exchange time of all the clients (or $\tau_{t,n}$) is not known to us before real scheduling. In the next subsection, we will present a $\rm C^2MAB$ estimation to conquer such a barrier.
\subsection{Estimation of Model Exchange Time with $\rm \mathbf{C^2MAB}$ }
\subsubsection{Background knowledge on $\rm \mathbf{C^2MAB}$ and UCB}
Each round selection in a Contextual Combinatorial Multi Arm Bandit ($\rm C^2MAB$) is characterized by a tuple $\left( \mathcal{N},\mathcal{S}_t, \{\bm\theta_{n}^*\}_{n \in \mathcal{N}} ,\{\bm c_{t,n}\}_{n \in \mathcal{N}}, \{\epsilon_{t,n}\}_{n \in \mathcal{N}}, f( \cdot ) \right)$, in which $\mathcal{N}$ represents the arm set and $\mathcal{S}_t$ is another set that catpures all the possible combination of arms. $\bm c_{t,n}^{}$ and $\bm \theta_{n}^*$ represents the contextual vector and coefficient vector respectively, among which, $\bm c_{t,n}$ is known before each round scheduling but dynamic between rounds, while $\bm \theta_{n}^*$ is unknown but stationary. After each round of scheduling, a combination of arms (often being called as a super arm) $S_t \subset \mathcal{S}_t $ is put into play. Then loss drawn from each selected arm, formulated by $ l_{t,n}=\bm c_{t,n}^{\top} \bm \theta_{n}^*+ \epsilon_{t,n}, n \in S_t $ is revealed to the scheduler, and meanwhile, a collective loss $ f( \{ r_{n,t} \}_{n \in S_t} ) $ is imposed. Our ultimate aim in the $\rm C^2MAB$ setting is to minimize the expected cumulative penalty $ \frac{1}{T}\sum_{t=1}^{T} \mathbb{E}\left[ f( \cdot ) \right]$ as far as possible by a careful selection on $S_t$. \par
Now we shall give a high-level description of a plausible solution for $\rm {C^2MAB}$, i.e., a UCB algorithm. The UCB algorithm takes the upper confidence bound as the optimistic estimation of the expected loss in each round. As the historical data accumulated, (i.e. $ l_{t,n}$ in the previous rounds), the bound could be narrowed and eventually converges to the real value, and thereby gaining more precision for our scheduling. By this means, the expected cumulative penalty could be minimized to the full extent with the increase of rounds of play.
\subsubsection{Application}
Recall that the information of model exchange time, or at least an estimated one, is supposed to be fetched before real client selection. One can take advantage of a $\rm {MAB}$ based technique to predict the model exchange time for all clients based on their historical performance (or to say, their reputation). In particular, each client can be regarded as an arm$\footnote{We use an arm to represent a specific client in our later analysis.}$ in a bandit setting and a combination of them (i.e. a super arm) is put into training, after which, the model exchange time for the selected arm, namely, $\{\tau_{t,n}\}_{n \in \mathcal{S}_t}$ can be observed by the scheduler.
 \par Normally, the model exchange time is associated with the client's computation capacity, running status as well as the bandwidth allocation for the model update. In this regard, we consider introducing linear contextual bandit into our estimation. Formally, we let $\bm c_{t,n} \triangleq [ 1/ \mu_{t,n},s_{t,n},M/B_{t,n}]^{\top}$ denote the contextual feature vectors that are collected by the scheduler before the scheduling phase. More explicitly, $\mu_{t,n}$ is the ratio of available computation capacity of client $n$ over round $t$. We can simply comprehend $\mu_{t,n}$ as the available CPU ratio of the client\footnote{ Note that $\mu_{t,n}$ could exceed 100\% since a client could have more than 1 CPUs, say, $\mu_{t,n}=200\%$ when 2 CPUs are free.}. A binary indicator $s_{t,n}$ indicates if client $n$ has participated in training in the last round. $M$ is the size of the model's parameters (measured by bit) and $B_{t,n}$ indicates the allocated bandwidth. Barring the available computation capacity of clients (i.e. $\mu_{t,n}$), which have to be proactively reported by the clients, all the other information could be fetched by the servers with ease. Therefore, here we can just comprehend the contextual feature $\bm c_{t,n}$ as some prior information known by us before we do the scheduling. Given the contextual features, we assume that the sampling value of $\tau_{t,n}$ complies with the following equation:
 \begin{align} \label{time model} \tau_{t,n}= \bm c_{t,n} ^{\top} \bm 	\theta_{n}^*  +\epsilon_{t,n}\end{align}
where $\bm \theta_{n}^* \triangleq [\tau_n^{b},\tau_n^s,1/\eta]^{\top}$ captures the static coefficient factors that are presumed to be unknown to the scheduler as they are hard to be detected by the server or even by the clients themselves. More explicitly, $\tau_n^{b}$ is the local training time for 100\% computation capacity. Multiplying it with the first element in $\bm c_{t,n}$, we get the approximated local training time under the computation capacity provided by clients.
$\tau_n^s$ denotes the cold start time, multiplying which with the second element $s_{t,n}$ in contexts yields the real data preparation time. This formulation is derived from the fact that clients who did not undertake the previous round of training need to spend extra time for data preparation, say, loading the data into memory. Likewise, we let $\eta\triangleq \log(1+\text{SNR})$ and multiplying which with $B_{t,n}$ yields the Shannon formula that we use to calculate the uploading data rate. Here SNR is an abbreviation of Signal-to-Noise Ratio, which is associated with the client profile (e.g. transmission power and channel condition). In this regard, $M/(B_{t,n}\eta)$ can fully represent the model uploading time for client $n$. In light of our formulation, $\bm c_{t,n} ^{\top} \bm \theta_n^*$ yields the approximation of the expected model exchange time.\par
 In addition, acknowledging some deviation, we admit a noise factor $\epsilon_{t,n}$ in our estimation, which is assumed to be a zero-mean random variable, conditionally sampling from an unknown distribution with left-bounded support, i.e. Supp$(\epsilon_{t,n}| \bm c_{t,n} ^{\top})= (a,b]$ where $a>-\bm c_{t,n} ^{\top} \bm \theta_{n}^*$ and $b$ is arbitrary. This assumption is made to ensure that $\tau_{t,n}$ must be always positive. Also, we have to make sure that $\epsilon_{t,n}$ is conditionally $R$-sub-Gaussian where $R \geq 0$ is a fixed constant. Formally, we need:
\begin{equation}\forall \Lambda \in \mathbb{R} \quad \mathbb{E}\left[e^{\Lambda \epsilon_{t,n}} \mid \bm c_{1: t,n}, \epsilon_{1: t-1,n}\right] \leq \exp \left(\frac{\Lambda^{2} R^{2}}{2}\right) \end{equation}
This assumption is necessary for the regret analysis of a linear bandit, which is also adopted by \cite{abbasi2011improved}. Though we admit some loss of generality for the noise assumption, we argue that a great number of distribution families in nature corresponds to $R$-sub-Gaussian (e.g. any distributions with zero mean bounded support, zero-mean Gaussian distribution, etc), so the assumption would not compromise the objectivity of this paper. \par
Now we let $\tau_{t,n}^*=\mathbb{E}[\tau_{t,n}]=\bm c_{t,n} ^{\top} \bm 	\theta_{n}^*$. If $\tau_{t,n}^*$ is clearly known to us, we can safely substitute $\tau_{t,n}$ in \textit{P3} with it.  Recall that $\bm 	\theta_{n}^*$ is an inherent feature of each arm (or client) that is supposed to be static, unchangeable over time. With this assumption, although the scheduler has no access to the real value of $\bm 	\theta_{n}^*$, which creates a barrier in the calculation of $\tau_{t,n}^*$, this value can be predicted using the historical information (or the reputation of an arm). For such a linear formulation, ridge regression could suit well. Now we let  $(\mathbf{D}_{t,n},\mathbf{y}_{t,n})$ to represent $p$ pieces of client $n$'s historical performance (i.e. the previous model exchange time and the contexts) that are obtained before round $t$. Formally, we have:
\begin{equation}
\mathbf{D}_{t,n}=\left[\begin{array}{ll}
{\mathbf{c}_{n}^{(1)}  }\\
\vdots\\
{\mathbf{c}_{ n}^{(p)}}\ 
\end{array}\right]_{p \times 3} \quad \mathbf{y}_{t,n}=\left(\begin{array}{cc}
{\tau_{n}^{(1)}}  \\ 
\vdots\\
{\tau_{n}^{(p)}} 
\end{array}\right)
\end{equation}
where $\mathbf{c}_{ n}^{(p)}$ and $\tau_{n}^{(p)}$ respectively represent the context and the real model exchange time of the $p$-th play of the arm $n$. With ridge regression, we can empirically estimate $\bm 	\theta_{n}^*$ with $\hat{\boldsymbol{\theta}}_{t,n}$:
\begin{equation}
\hat{\boldsymbol{\theta}}_{t,n}=\left(\mathbf{D}_{t,n}^{\top} \mathbf{D}_{t,n}+\lambda \mathbf{I}_{3}\right)^{-1} \mathbf{D}_{n}^{\top} \mathbf{y}_{t,n}
\end{equation}

For ease of algorithm's design, we then transform $\hat{\boldsymbol{\theta}}_{t,n}$ into an equivalent form, as follows:
\begin{equation}
\label{theta prediction}
\hat{\boldsymbol{\theta}}_{t,n}=\mathbf{H}_{t-1,n}^{-1} \mathbf{b}_{t-1,n}
\end{equation}
where $\mathbf{H}_{T,n}=\mathbf{H}+\sum_{t=1}^{T} x_{t,n} \mathbf{c}_{t,n}\mathbf{c}_{t,n}^{\top}$ and  $\mathbf{b}_{T,n}=\sum_{t=1}^{T} x_{t,n} \tau_{t,n} \mathbf{c}_{t,n}$. Among which, $\mathbf{H}=\lambda \mathbf{I}$.
\par
As we are going to take advantage of the UCB algorithm we previously discussed as our solution, we resort to $\bar{\tau}_{t,n}$ as the optimistic estimation of ${\tau}_{t,n}$, which has the following form:
\begin{equation}
\label{UCB estimation}
\bar{\tau}_{t,n}\triangleq \max\left\{\bm c_{t,n} ^{\top}  	\hat{\boldsymbol{\theta}}_{t,n}-\alpha_{t} \sqrt{\mathbf{c}_{t,n}^{\top} \mathbf{H}_{t-1,n}^{-1} \mathbf{c}_{t,n}},0\right\}
\end{equation}
where $\alpha_{t}$ is an exploration parameter.\par
Now we show in Lemma  \ref{confidence bound} the validity of  the given confidence bound (i.e. to show the real expected exchange time does not deviate much from the confidence bound with a high probability).
\begin{lemma}  \label{confidence bound}  If we set $\alpha_{t}=R\sqrt{3 \log \left(\frac{1+t  L^2/ \lambda}{\delta}\right)}+\lambda^{1 / 2} S$,
with probability at least $1-\delta,$ we have
\begin{equation}
0 \leq \tau_{t,n}^{*}-\bar{\tau}_{t,n} \leq 2 \alpha_{t}\left\|\mathbf{c}_{t,n}\right\|_{\mathbf{H}_{t-1,n}^{-1}}
\end{equation}
for any round $t \geq 1$ and any arm $n\in \mathcal{N}$
\end{lemma}
\begin{proof}
The complete proof is given in Appendix \ref{Proof2}.
\end{proof}
 We first note here that Lemma \ref{confidence bound} will be used in our analysis of regret bound, which will be shown in the next section. \par
As we have decided $\bar{\tau}_{t,n}$ as our estimation of $\tau_{t,n}$, we now transfer \textit{P3} to the ultimate form, shown in the following:
  \begin{align}
\begin{split}
\textit{(P4)}:&  \min_{  \mathbf{x}_t} \quad   V \max_{n \in \mathcal{N}}\{ x_{t,n} \bar{\tau}_{t,n}  \}-\sum_{n\in \mathcal{N}}Z_{t,n} x_{t,n}\\
s.t. \quad&  \sum_{n\in \mathcal{N}} x_{t,n}=\min \left \{ m, \sum_{n\in \mathcal{N}} I_{t,n} \right \}\\
& x_{t,n} \leq I_{t,n}\\
& x_{t,n}  \in \{0,1\}
\end{split}
\end{align}
Then transformed problem is an Integer Linear Programming (ILP) problem, which is indeed solvable and for which we design a divide-and-conquer-based algorithm for an efficient settlement, shown in the coming section.
\section{Algorithms and Analysis}
In this section, we first present the detail of our proposed algorithm, and then some related analysis is given.
\subsection{Algorithms Design}
\begin{algorithm}[!hbtp]  
        \caption {Divide-and-conquer solution for \textit{P4}} 
        \begin{algorithmic}[1] 
        \REQUIRE ~~\\
        The estimated time for model exchange; $ \{\bar{\tau}_{t,n}\}_{n \in \mathcal{N}}$\\
        The expected number of chosen arms; $m$\\
        Indicator function of arms' availability; $\{I_{t,n}\}_{n \in \mathcal{N}}$\\
        Length of virtual queue;  $ \{Z_{t,n}\}_{n \in \mathcal{N}}$
        \ENSURE~~\\
        The solution for \textit{P4} in round $t$; $\{x_{t,n}\}_{n \in \mathcal{N}}$
      \STATE Set  $\boldsymbol{Z}^*_t=\{Z_{t,n}\}_{I_{t,n}= 1}$ 
      \STATE Use $\mathcal{A}_t$ to store arms with an descending order  of $\boldsymbol{Z}^*_t$
      \STATE Use $\mathcal{N}^+_t$ to store all the $n$ that satisfies $I_{t,n}= 1$
      \STATE Set $k=\min\{ m, \sum_{n\in \mathcal{N}} I_{t,n} \}$ // \# of clients to be picked
	\FOR{$n_{max}\in \mathcal{N}^+_t$ }
	\STATE Initialize an empty set $\mathcal{S}_{n_{max}}$
	\FOR{$n \in \mathcal{A}_t$ }
	\IF{$\bar{\tau}_{t,n}\leq\bar{\tau}_{t,n_{max}}$}
	\STATE Push $n$ into  $\mathcal{S}_{n_{max}}$
	\ENDIF
	\IF{$length(\mathcal{S}_{n_{max}}) == k$}
	\STATE Calculate the objective of \textit{P4} as $F_{n_{max}}$ based on $\mathcal{S}_{n_{max}}$
	\STATE Break the first loop
	\ENDIF
	\ENDFOR
	\ENDFOR
	\STATE Set $n^*$ the index of minimum $F_{n_{max}}$ among those being calculated in line 12.
	\STATE Return $\{x_{t,n}\}$ that represented by $\mathcal{S}_{n^*}$ 
        \end{algorithmic} 		
	\end{algorithm}
 Noticeably, the first term on the objective function of \textit{P4} has only finite possible values, so we can simply iterate these values and transform them into the constraint in the sub-problems. By this means, we divide the problem into a few smaller-scale sub-problems, which are easier to conquer.
Formally, the sub-problem after division is shown in the following:
 \begin{align}
\begin{split}
\textit{(P4-SUB)}:&  \min_{ \mathbf{x}_t} \quad  -\sum_{n\in \mathcal{N}}Z_{t,n} x_{t,n}\\
s.t. \quad&  \sum_{n\in \mathcal{N}} x_{t,n}=\min \left \{ m, \sum_{n\in \mathcal{N}} I_{t,n} \right \}\\
&x_{t,n} \bar{\tau}_{t,n} \leq \bar{\tau}_{max} \\
& x_{t,n} \leq I_{t,n}\\
& x_{t,n}  \in \{0,1\}
\end{split}
\end{align}
where $\bar{\tau}_{max}$ is one of the fixed value among the possible values of the first term in \textit{P4}. \textit{P4-SUB} is much easier to conquer. First we only need to filter those qualified clients with a smaller or equal $\bar{\tau}_{t,n}$ to $\bar{\tau}_{max}$, and with an active status (or to say $I_{t,n}=1$). Trivially, the sub-problem can be solved by finding $k=\min\{ m, \sum_{n\in \mathcal{N}} I_{t,n} \}$ clients with the biggest $Z_{t,n}$ among the qualified clients. After the divide-and-conquer process, we only need to compare all the objectives obtained from the sub-problems and select the minimum one as our final achieved solution. The detail of the above process can be found in Algorithm 1, which could at least reach a computation complexity of $\mathcal{O}(N^2)$.
	
\begin{algorithm}[!hbtp]  
        \caption {Reputation Based Client Selection with Fairness (RBCS-F)} 
        \begin{algorithmic}[1] 
        \REQUIRE ~~\\
        The expected number of involved clients each round; $m$ \\
        Exploration parameter; $\alpha_0, \alpha_1,\dots$\\
	 The set of clients; $\mathcal{N}$, Parameter for ridge regression; $\lambda$\\
	 The guaranteed participating  rate; $\beta$\\
	 Parameter for objective balance; $V $\\
        \ENSURE~~\\
        The control policy $\pi=\{x_{t,n}\}_{n \in \mathcal{N}, t=0,1, \dots}$
        \FOR{$n \in \mathcal{N}$ }
         \STATE  Initialize $\mathbf{H}_{0,n} \leftarrow \lambda \mathbf{I}_{3 \times 3}, \mathbf{b}_{0,n} \leftarrow \mathbf{0}_{3}^{\top}$, $Z_{0,n}\leftarrow0$
        \ENDFOR
	\FOR{$t= 1,2\dots$ }
	\STATE Observe current contexts $\{\mathbf{c}_{t,n}\}$ and arms availability $\{I_{t,n}\}$
		\FOR{$n \in \mathcal{N}$ }
			\STATE $\hat{\boldsymbol{\theta}}_{t,n} \leftarrow \mathbf{H}_{t-1,n}^{-1} \mathbf{b}_{t-1,n}$
			\STATE$\hat{\tau}_{t,n} \leftarrow\mathbf{c}_{t,n}^{\top} \hat{\boldsymbol{\theta}}_{t,n} $
			\STATE $\bar{\tau}_{t,n} \leftarrow \hat{\tau}_{t,n}-\alpha_{t} \sqrt{\mathbf{c}_{t,n}^{\top} \mathbf{H}_{t-1,n}^{-1} \mathbf{c}_{t,n}}$
		\ENDFOR
		\STATE // Execute Algorithm 1 for a decision \\ $\{x_{t,n}\} \leftarrow  \text{Algorithm 1}(\{\bar{\tau}_{t,n}\}, m, \{ I_{t,n}\}, \{ Z_{t,n}\})   $
		\STATE Distribute model to the selected clients and observe their model exchange time;$\{\tau_{t,n}\}$
		\FOR{$n \in \mathcal{N}$ }
		\STATE Update $Z_{t,n}$ according to  (\ref{Z queue})
	\STATE $\mathbf{H}_{t,n} \leftarrow \mathbf{H}_{t-1,n}+x_{t,n} \mathbf{c}_{t,n} \mathbf{c}_{t,n}^{\top}$
	\STATE $\mathbf{b}_{t,n} \leftarrow \mathbf{b}_{t-1,n}+ x_{t,n} \tau_{t,n}  \mathbf{c}_{t,n}$
	\ENDFOR
	\ENDFOR
        \end{algorithmic} 		
	\end{algorithm}
With Algorithm 1 introduced, now we are going to discuss our proposed solution for fairness-aware FL, termed Reputation Based Client Selection with Fairness (RBCS-F), shown in Algorithm 2.
The working procedure of RBCS-F is quite intuitive. The algorithm starts with initialization of some parameters in the first three lines, and then begins to start iterative federated learning. In every iteration, the scheduler first observes the contexts and the availability of the arms (i.e. FL clients), then estimates the model exchange time with Eqs. (\ref{theta prediction}) and (\ref{UCB estimation}) using historical information. Taking advantage of the observed context, availability as well as the estimation, the selection scheme for this round could be fetched by Algorithm 1. After the decision, the model would be distributed to the selected clients and gathered after local training. Before the end of a round, the algorithm records the exchange time of the selected clients and update the associated parameters, as shown in lines 14-16.
\subsection{ Theoretical Analysis }
\subsubsection{Regret and Fairness Guarantee}
In an MAB model, regret is a key performance metric that measures the performance gap between a given policy and the optimal policy. Therefore, for ease of analysis, we first define the time average regret of RBCS-F.
\begin{definition}
Time average regret of RBCS-F is defined as:
\begin{equation}
R(T) \triangleq \frac{1}{T}\sum_{t=1}^{T} \mathbb{E}\left[   {f}(\mathcal{S}_t,\bm \tau_t)- {f}(\mathcal{S}^*_t,\bm \tau_t)\right] 
\end{equation}
where we leverage $\mathcal{S}^*_t$ to represent the decision made by the optimal policy while $\mathcal{S}_t$ captures RBCS-F's decision.
\end{definition}
To proceed, we show a strict upper bound on time average regret of RBCS-F in Theorem \ref{regret bound}.
\begin{theorem}
\label{regret bound}
 Given any control parameter $V$, with probability at least $(1- \delta)^2$, the time average regret achieved by RBCS-F is upper bounded by:
\begin{align}
\begin{split}
R(T)\leq&  \frac{N\left(1+\beta^2 \right)}{2V}+ \zeta_T \sqrt{ \frac{6 \log ( 1+T L^2/ 3 \lambda)}{T} }
\end{split}
\end{align}
where $S$ and $L$ are both positive finite constants satisfying $\left\|\bm\theta_{n}^{*}\right\|_{2} \leq$
$S$ and $\left\|\mathbf{c}_{t,n}\right\|_{2} \leq L$ for all $t \geq 1$ and $n \in \mathcal{N}$. And: \par
 $\zeta_T=\max\{K,1\} \cdot \max\left\{2R \sqrt{3 \log \left(\frac{1+T  L^2/ \lambda}{\delta}\right)}+\lambda^{1 / 2} S,1 \right\}$
 where $K$ is a constant value. 
\end{theorem}
\begin{proof}
The complete proof is given in Appendix \ref{Proof3}.
\end{proof}
Now we give another theorem to ensure that the long-term fairness constraint would not be violated.
\begin{theorem}
\label{mean rate stable}
For RBCS-F, the fairness vitual queues are all mean rate stable in any setting of $V$, thus the time average fairness is being guaranteed.
\end{theorem}
\begin{proof}
The complete proof is given in Appendix \ref{Proof4}.
\end{proof}
\subsubsection{Impact of $V$}
In light of Theorem \ref{regret bound}, it seems quite reasonable for us to set the penalty factor $V$ as large as possible so as to eliminate the first term in the regret upper bound. Such an extreme setting seems even more attractive regarding the fact that the long-term fairness constraint holds under any setting of $V$, which is justified by Theorem \ref{mean rate stable}. Although a large value of $V$ could indeed bring us a more satisfying long-term model exchange time while satisfying the long-term fairness constraint, we must claim here that the fairness factor is \textit{not} impervious to the setting of $V$. Note that our long-term fairness constraint is built on the premise that the training rounds are infinite, but this may not be true in real training. With a larger $V$, the fairness queue will have a slower rate to converge, indicating that fairness could not be well guaranteed before convergence. When the training rounds are finite, the number of rounds that need to undergo before convergence could compromise some degrees of fairness. Such an analysis could be verified by our experiment results that we are now going to display.
\section{Experiments}
In this section, we present the detail of our experiments. In the first sub-section, we would explain the general setting of our simulation environment and evaluate the numerical performance of our proposed solutions. The numerical evaluation results could well explain the relationship between the penalty factor ($V$), fairness (reflected by the queue length), and efficiency guarantee (the time span of a federated round). Then we will move on to the evaluation of the real training of two iconic public datasets, CIFAR-10 and fashion-MNIST, both of which are evaluated under different settings of non-iid extent. The real-data experiment will show how our proposed RBCS-F impacts the training efficiency and final model performance (i.e. accuracy).
\subsection{Numerical Simulation}
\subsubsection{Simulation setting}
In our simulation, we assume the model exchange time conforms to the linear formulation as shown in Equation (\ref{time model}). To simulate a heterogeneous system with clients of different computation and communication capacity, we equally divide the total number of 40 clients into 4 classes and accordingly endow disparate abilities to them. For clearness, one can check Table \ref{setting of theta} for the inherent training setting of different classes of clients.
\par For the context generation (in order to simulate the per-round status of clients), we assume the allocated bandwidth of all clients is sampling from a uniform distribution between $[2,4]$MHz and the model size M is fixed to $20$Mb. Likewise, the available computation capacity of all clients is also sampling from the same uniform distribution within $[50\%,200\%]$. The indicator $s_{t,n}$ is set according to the training decision in the last round. In addition, for the noise in our linear formulation, we draw $\epsilon$ from a conditional uniform distribution within $(-\bm c_{t,n} ^{\top} \bm \theta_n^*,\bm c_{t,n} ^{\top} \bm \theta_n^*)$. The availability of clients follows the same Bernoulli distribution with parameter 0.8, and the setting of other algorithm related parameters could be found in Table \ref{Parameters setting}.  In our simulation, we mainly compare RBCS-F with two baseline selection methods that are commonly used in the field, i.e. random and FedCS\cite{nishio2019client}. Note that we have made an adaption to FedCS in order to accommodate it to our context, but the basic idea is the same as the vanilla one, which is to select as much as clients within a fixed deadline. More concretely, we allow FedCS to have full access to both the contextual features and the static coefficient factor. With the additional information, its strategy is to select all the clients that possess an expected training time (i.e. $\bm c_{t,n} ^{\top} \bm \theta_{n}^*$) that shorter than the pre-set deadline. 
\if\arxiv0 \begin{spacing}{0.3}  \fi
\begin{table}
 \small
\caption{Inherent setting of arms (or clients)}
\label{setting of theta}
\centering
\begin{tabular}{cccc}
\hline
 client & $\tau_n^{b}$ & $\tau_n^s$ &$\eta$ \\ 
  class & & (cold start time) &$\log(1+SNR)$ \\ \hline
1 &1s&1s&$\log(1+1000)$\\\hline
2 &2s &1s&$\log(1+100)$\\\hline
3&3s &1s&$\log(1+10)$\\\hline
4 &4s &1s&$\log(1+1)$\\\hline
\end{tabular}
\end{table}
\if\arxiv0 \end{spacing}\fi

\if\arxiv0 \begin{spacing}{0.3}\fi
\begin{table}
 \small
\caption{Parameters setting}
\label{Parameters setting}
\centering
\begin{tabular}{ccc}
\hline
notation& meaning & value\\  \hline
$\beta$ &guaranteed participating rate &0.15\\\hline
$m$ &maximum selected clients &8\\\hline
$\lambda$&parmeters for ridge regression &1\\\hline
$\alpha_{t}$ &exploration factor &0.1\\\hline
\end{tabular}
\end{table}
\if\arxiv0 \end{spacing}\fi
\subsubsection{Numerical performance evaluation}
In our first evaluation, we show the variation of queue status for RBCS-F under different values of penalty factor $V$. As shown in Fig. \ref{queue status}, where RBCS-F($x$) is abbreviated for RBCS-F with a penalty of $V=x$, it is interesting to see that all the curves with different settings of $V$ flatten after going through a number of scheduling rounds. This phenomenon can justify our conclusion of the mean rate stability of the queues, which indicates that they could not grow to infinity and break our fairness constraint. Another observation we can derive here is that the curve with a higher penalty factor (i.e. $V$) seems to have a slower convergence speed and a higher convergence value. This implies that a large value of $V$ might sacrifice a few fairness before its convergence, although it does conform to the long-term fairness constraint. Such an observation is consilient with our explanation given in the remark below Theorem \ref{fairness queue theorem} and our theoretical analysis in the last section.
{\vspace{-0.2cm}
\begin{figure}[!hbtp]
\centering
\setlength{\abovecaptionskip}{0.cm}
\includegraphics[width=2.5in]{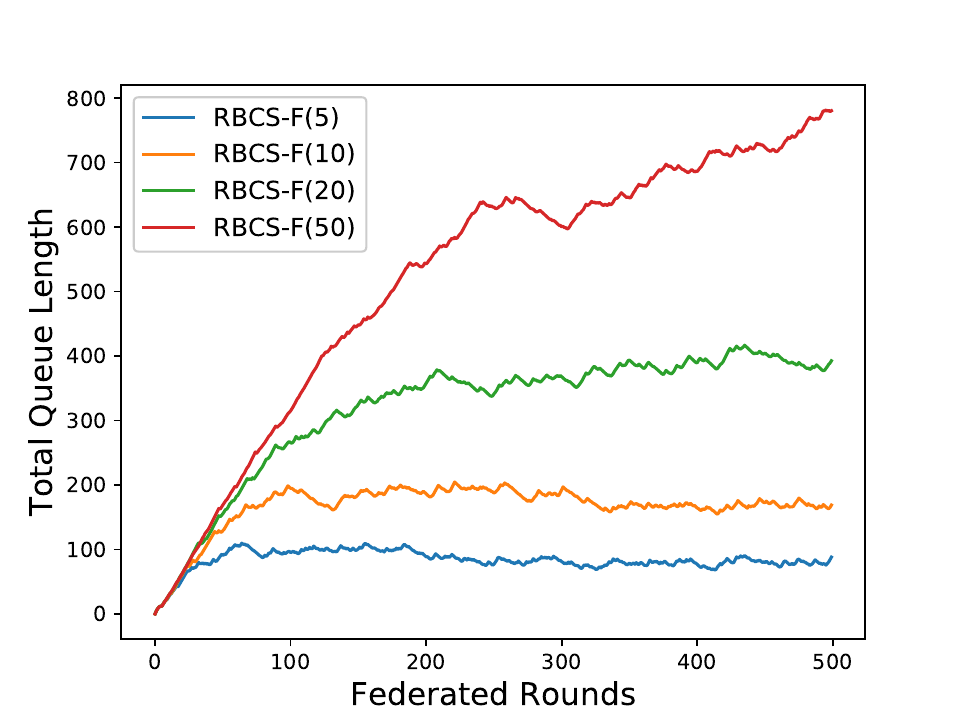}
\caption{The impact of $V$ on the convergence of queues}
\label{queue status}
\end{figure}
}
\par
Now we take a look at the evolution of training time across scheduling rounds. In Fig. \ref{time_training}, we depict the time consumption of our proposed RBCS-F with different $V$, and that of the random strategy and FedCS(3) \footnote{For FedCS(3), we set its deadline (one of its key parameter) to 3s. The specific setting allows us to make the number of its selection clients approximates to 8, which is exactly the selection number of other strategies (see our setting in Table \ref{Parameters setting}). }. As depicted, RBCS-F seems to have a satisfying enhancement in reducing the training time, compared with the random scheme, and of the same number of federated rounds, RBCS-F with a higher $V$ boasts a shorter time consumption. In addition, it is interesting to see that there is a performance gap between RBCS-F and FedCS(3). We note that this gap is inevitable due to our introduction of the fairness factor and the cost of online learning, but the bound itself is well-defined by our analysis of the regret. 
\begin{figure}[h]
\setlength{\abovecaptionskip}{0.cm}
\centering
\includegraphics[width=2.5in]{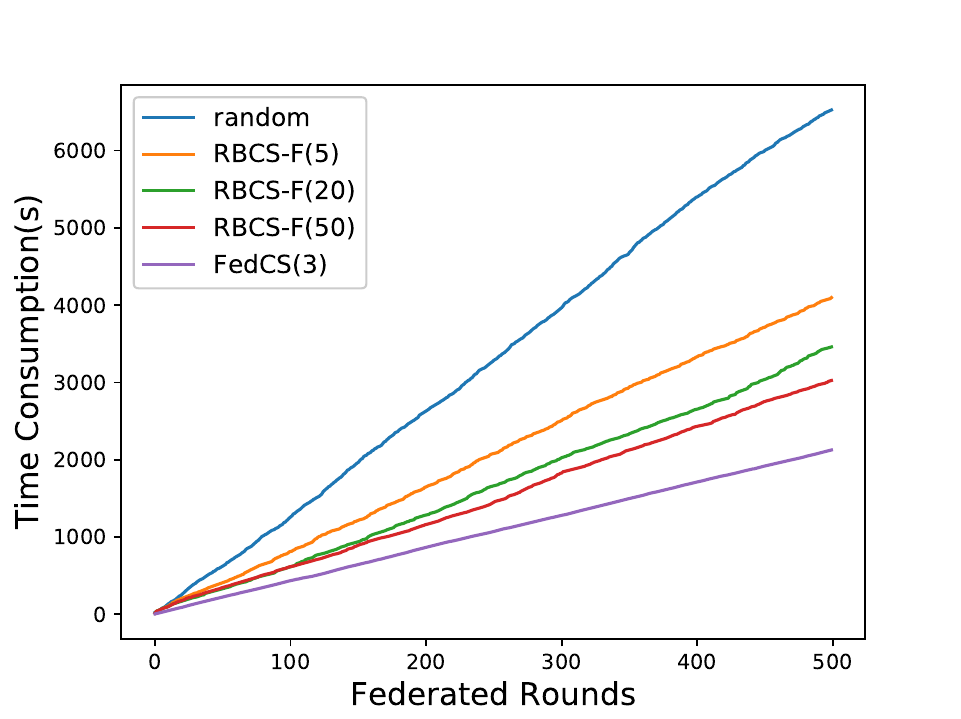}
\caption{Training time of different client-selection strategies}
\label{time_training}
\end{figure}
\par
\begin{figure}[!hbtp]
\setlength{\abovecaptionskip}{0.cm}
\centering
\includegraphics[width=2.5in]{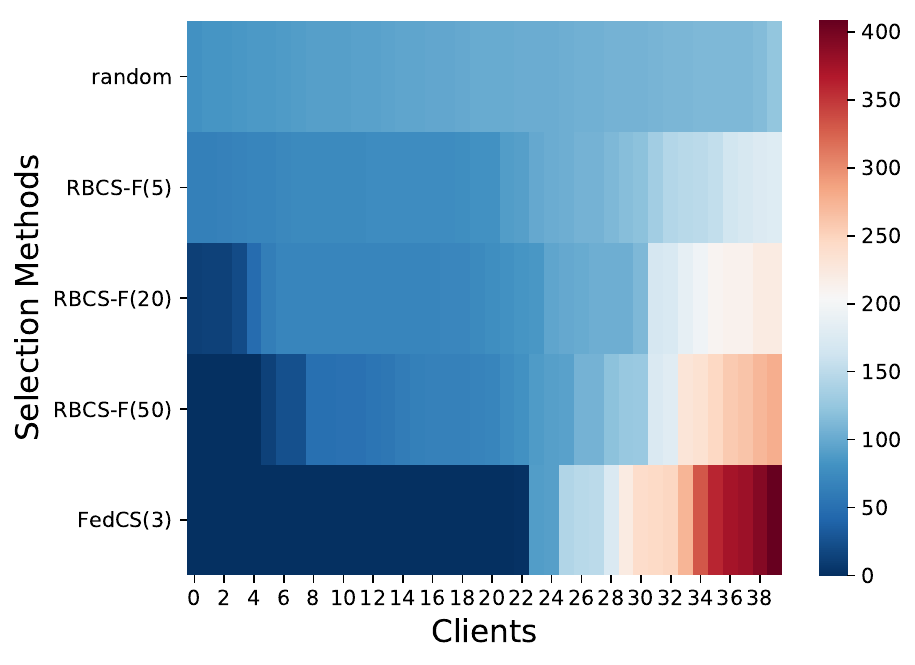}
\caption{Pull record of arms (or clients) under different client-selection strategies }
\label{pull record}
\end{figure}
The variance in training time of different schemes could be alternatively explained by looking at Fig. \ref{pull record}. This figure depicts the pull number of different arms (or chosen times of FL clients) after going through 500 rounds of decision, in which the clients are sorted in ascending order over their pull number. The brighter color indicates a heavier pull (or more times being selected) on the corresponding arm. From Fig. \ref{pull record}, we notice that the pull number of clients could vary dramatically when $V$ is set to a high value and the unbalanced selection is more intense for FedCS(3). By contrast, the scheme that is known to be fairer (e.g. random or RBCF-F with low penalty) boasts an even distribution on the pull number, based on which we can explain why the training time of RBCS-F would escalate with a fairer selection. Clearly, the selection scheme that evenly chooses the clients shall never match up with those always choosing the fastest ones. However, is it the faster the better? Does fairness matter in real training? Now we are going to explore the answers with our real training on two public datasets.
\subsection{Training on Public Dataset}
\subsubsection{Setup} 
We set up federated environment with \textit{PyTorch} (version: 1.6.0) and all the computation is conducted using a high-performance workstation (Dell PowerEdge T630 with 2x GTX 1080Ti).  We have prepared two tasks for an evaluation purpose. To be specific, we use two different Convolutional Neural Network (CNN) models to predict the classifying results from two datasets, fashion-MNIST, and CIFAR-10. For fashion-MNIST, we adopt a CNN with two 5x5 convolution layers (the first with 20 channels, the second with 50, each followed with 2x2 max pooling), a
fully-connected layer with 500 units and ReLu activation, and finally a softmax output layer. For CIFAR-10, which is known to be a harder task,  we use another much heavier CNN model with two 5x5 convolution layers (each with 64 channels), also followed with 2x2 max pooling, two fully connected layers with respective 384 and 192 units, and finally a softmax output layer.\par
In addition to the general iid setting, we also explore the training performance on a non-iid one. Here we adopt the same approach as in \cite{hsu_measuring_2019} to synthesize non-identical client data. More specifically, we uniformly sample $q_i \times 500$ items from each of the classifying class, where $\bm q \triangleq (q_1,q_2,\dots,q_i)$ is drawn from a Dirichlet distribution, i.e., $\boldsymbol{q} \sim \operatorname{Dir}(\gamma_1 \boldsymbol{p})$. Here $\boldsymbol{p}$ is an all-1 10-dimension vector   \footnote{Both Cifar and fashion-Mnist have 10 targets (or classes)} and $\gamma_1$ is a \textit{concentration} parameter controlling the extent of  identicalness among clients, say, with $\gamma_1 \rightarrow 0$ each client holds only one class chosen at random (i.e. high degree of non-iid), conversely, all clients have identical access to all classes (i.e. approximates to iid) if $\gamma_1 \rightarrow \infty$.

\subsubsection{Impact of fairness}
\begin{figure*}[!hbtp]
\centering
\subfloat[$\gamma_1=1$]{\includegraphics[width=2.4in]{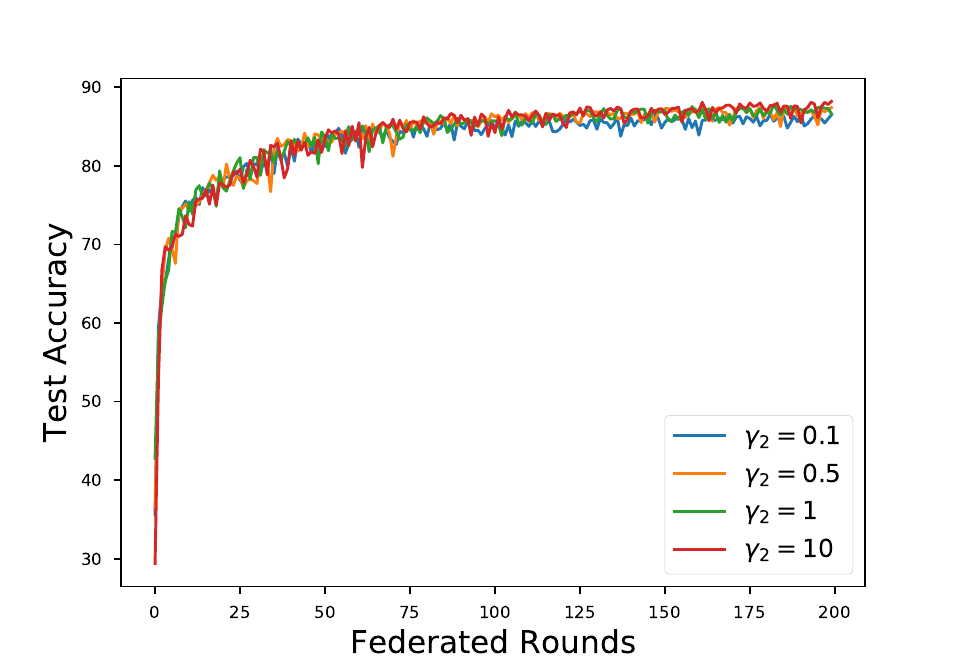}%
}
\subfloat[$\gamma_1=10$]{\includegraphics[width=2.4in]{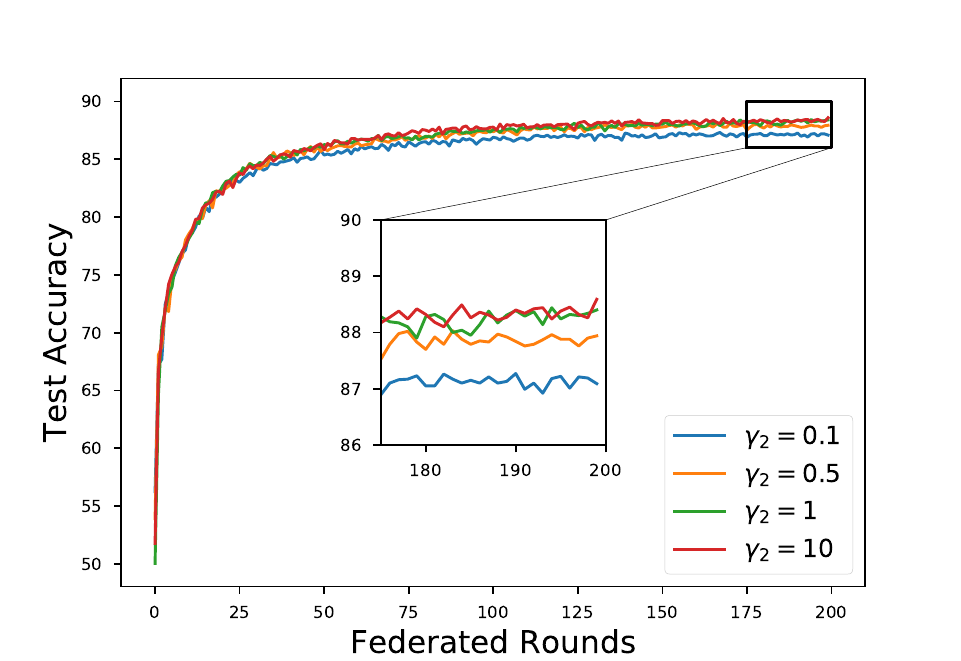}%
}
\subfloat[iid (approximates to $\gamma_1=\infty$ )]{\includegraphics[width=2.4in]{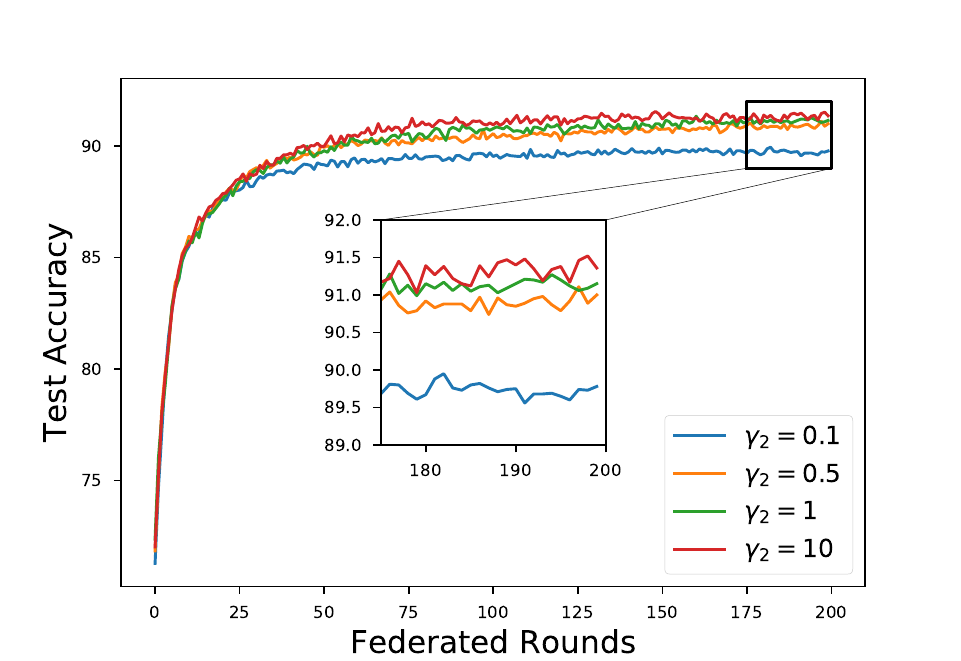}%
} \\
\subfloat[$\gamma_1=1$]{\includegraphics[width=2.4in]{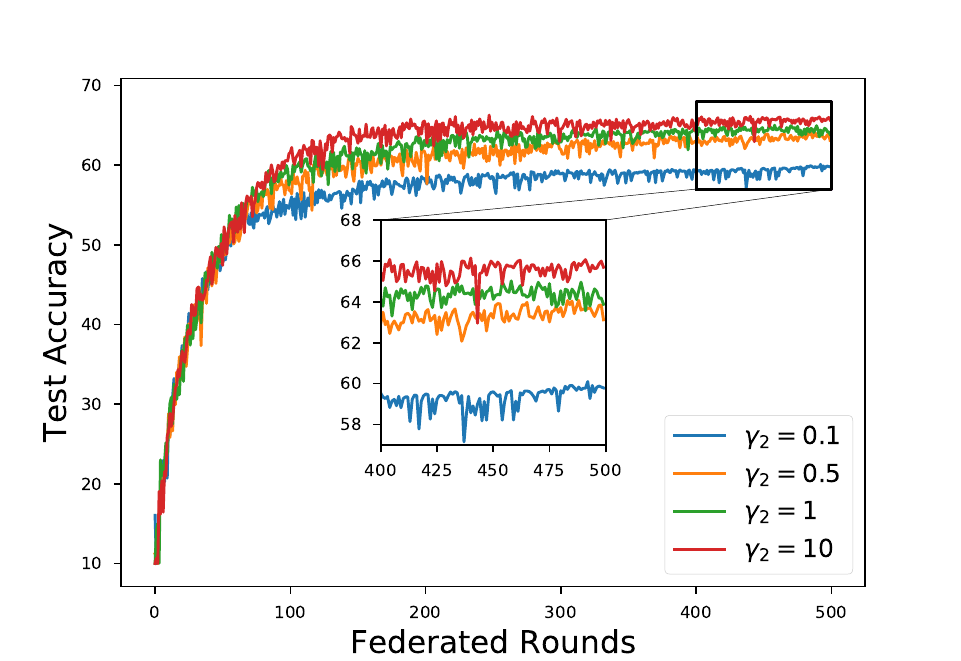}%
}
\subfloat[$\gamma_1=10$]{\includegraphics[width=2.4in]{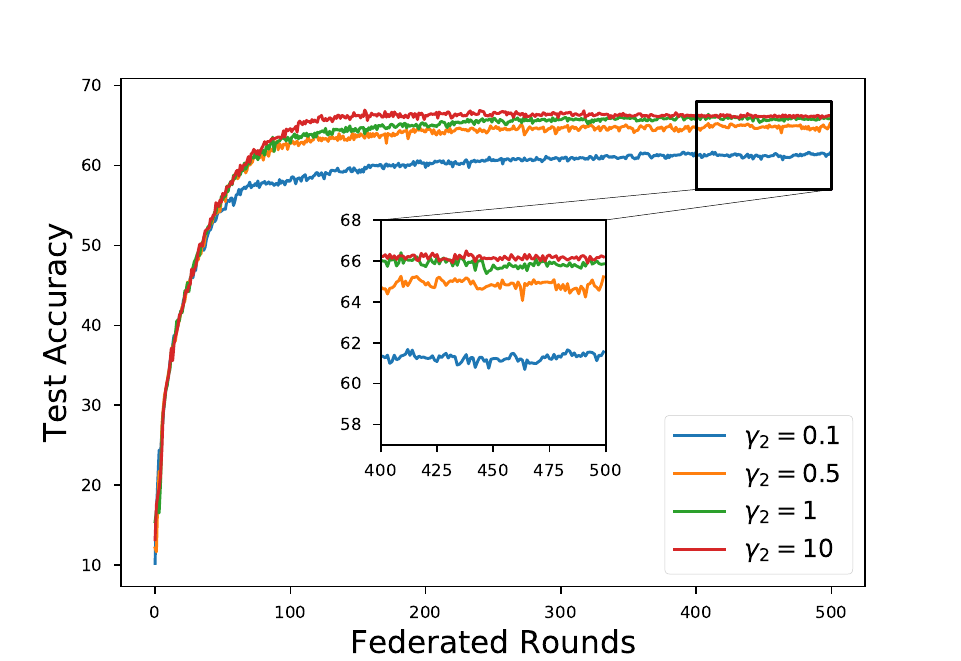}%
}
\subfloat[iid (approximates to $\gamma_1=\infty$ )]{\includegraphics[width=2.4in]{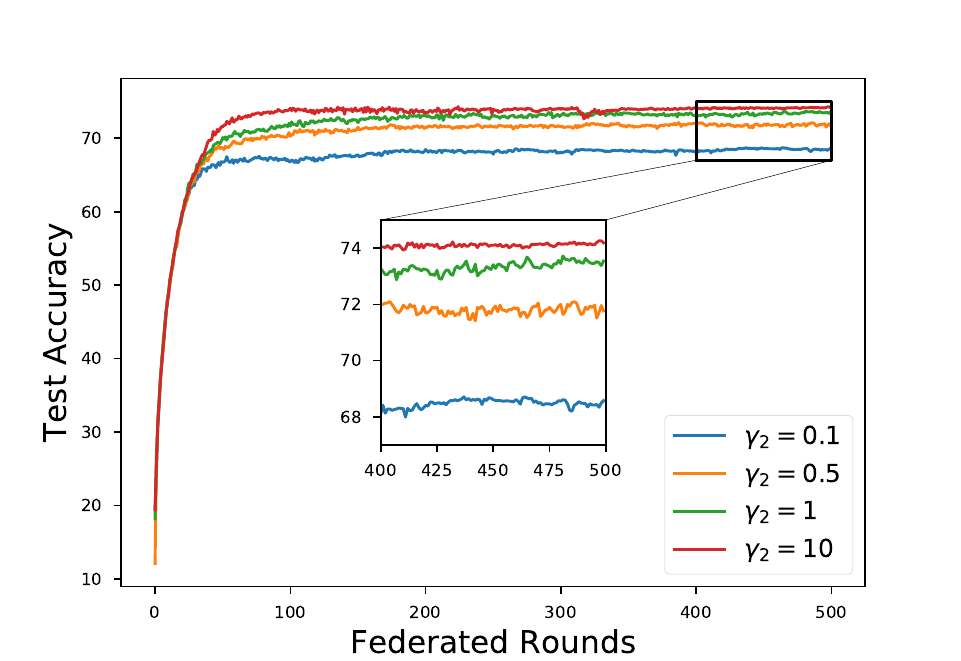}%
}
\caption{Fairness impact under fashion-MNIST ( (a), (b) and (c) ) and CIFAR-10 ( (d), (e) and (f) )}
\label{fairness impact}
\end{figure*}
In order to quantify the fairness factor and investigate how the factor affects the model accuracy as well as the training efficiency, we thereby introduce $\gamma_2$ to indicate the extent of fairness. Analogically to how we quantify the non-iid extent, we draw$\boldsymbol{q}$ from a Dirichlet distribution, i.e., $\boldsymbol{q}^{\prime} \sim \operatorname{Dir}(\gamma_2 \boldsymbol{p^{\prime}})$ and then $\boldsymbol{q}^{\prime} $ is serving as the probability vector that we use to randomly select clients in each round. As we note before, a smaller value of \textit{concentration} parameter $\gamma_2$ leads to a higher variation of $\boldsymbol{q}^{\prime}$ and thereby causing greater unbalance in selection. Fig. \ref{fairness impact} show how the model accuracy evolve with different $\gamma_2$, under different non-iid extent (given by $\gamma_1$). Among which, subfigures (a), (b),  and (c) depict that of the training for fashionMnist, where we can see that a higher $\gamma_2$ (a fairer selection) boasts a higher final model accuracy. Also, a similar observation, or an even more conspicuous one, can be found in our training for CIFAR-10, as indicated in subfigures (d), (e), and (f). From our result, it appears that the fairness factor might have different degrees of influence for the training of different datasets. More radically, we are in fact guessing that fairness factor would play a more critical role in a more complicated task. Our theory is that training of a harder task might require more diversified data (in terms of both targets and features), and corresponding, the relative information that each piece of data contains would reduce, and thereby, the training of those tasks should better involve as much available data as possible (i.e. better to be fair), so as to improve the model performance (specifically, final accuracy). \par
On the other hand, although the experimental data does demonstrate a profound impact of non-iid extent on the model stability and convergence speed during training (as we can observe in Fig. \ref{fairness impact} that when $\gamma_1$ decreases, more jitters on the curve and more rounds underwent before convergence), it does not explicitly show the fairness factor (as reflected by $\gamma_2$) being more or less influential with the change of $\gamma_1$, which seems to tell us that our defined non-iid extent  has little or no impact on the effect of fairness.

\par
%
\subsubsection{Accuracy vs. federated round}
\label{Accuracy vs. federated round}
\begin{figure*}[!hbtp]
\centering
\subfloat[$\gamma_1=1$]{\includegraphics[width=2.4in]{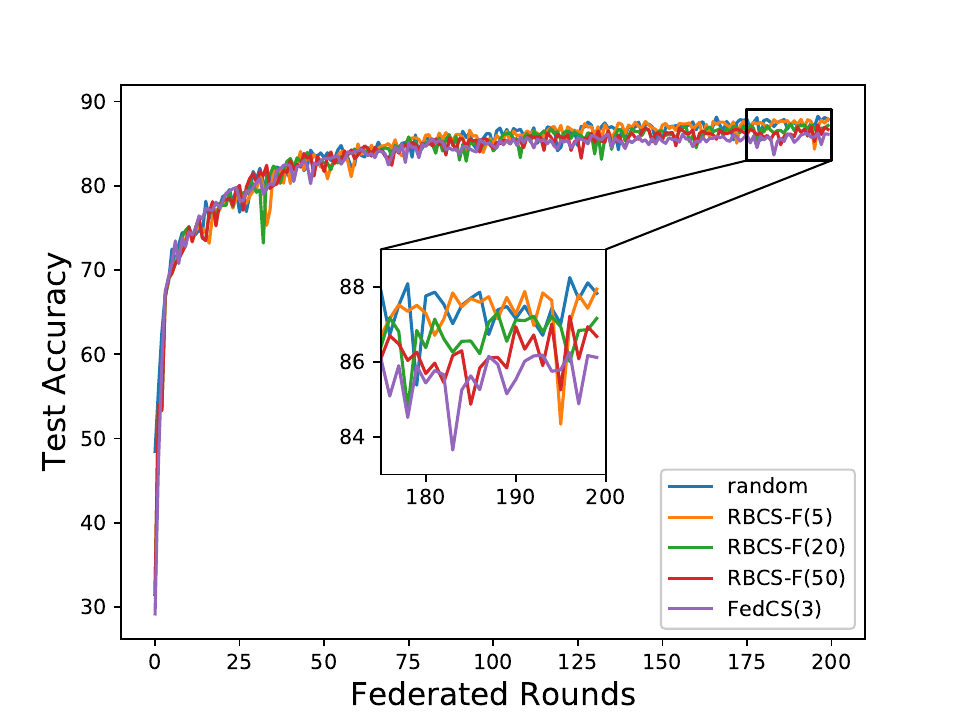}%
}
\subfloat[$\gamma_1=10$]{\includegraphics[width=2.4in]{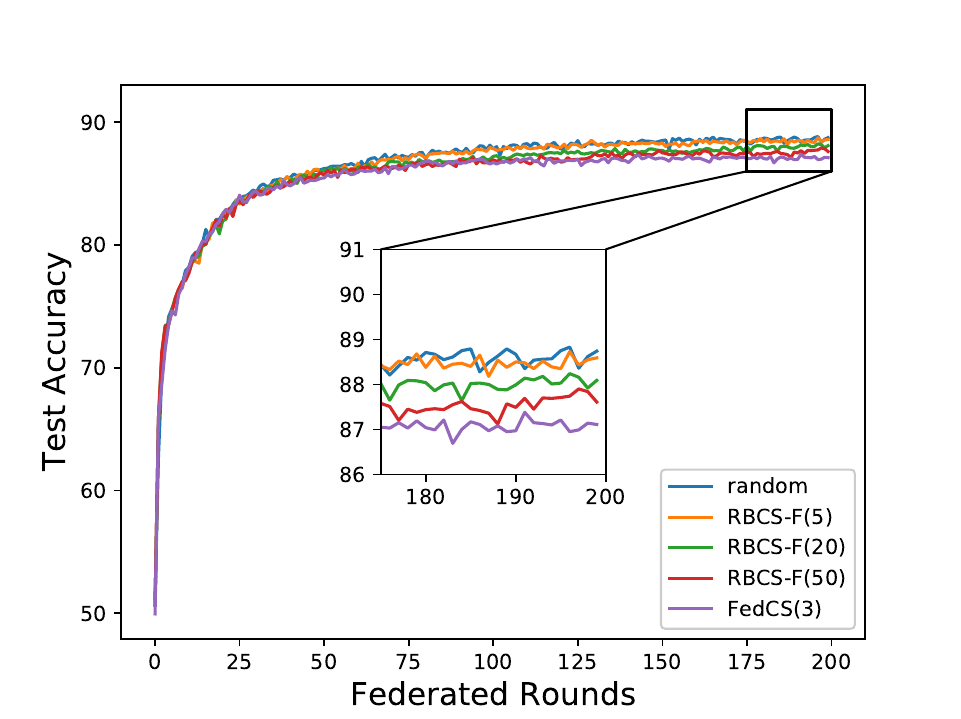}%
}
\subfloat[iid (approximates to $\gamma_1=\infty$ )]{\includegraphics[width=2.4in]{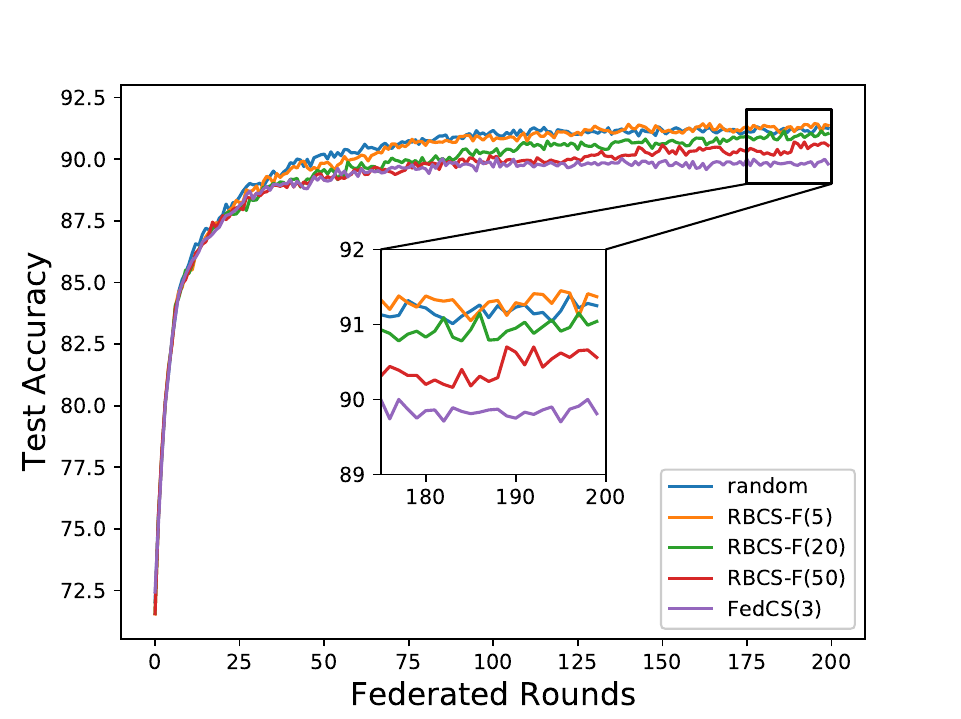}%
}\\
\subfloat[$\gamma_1=1$]{\includegraphics[width=2.4in]{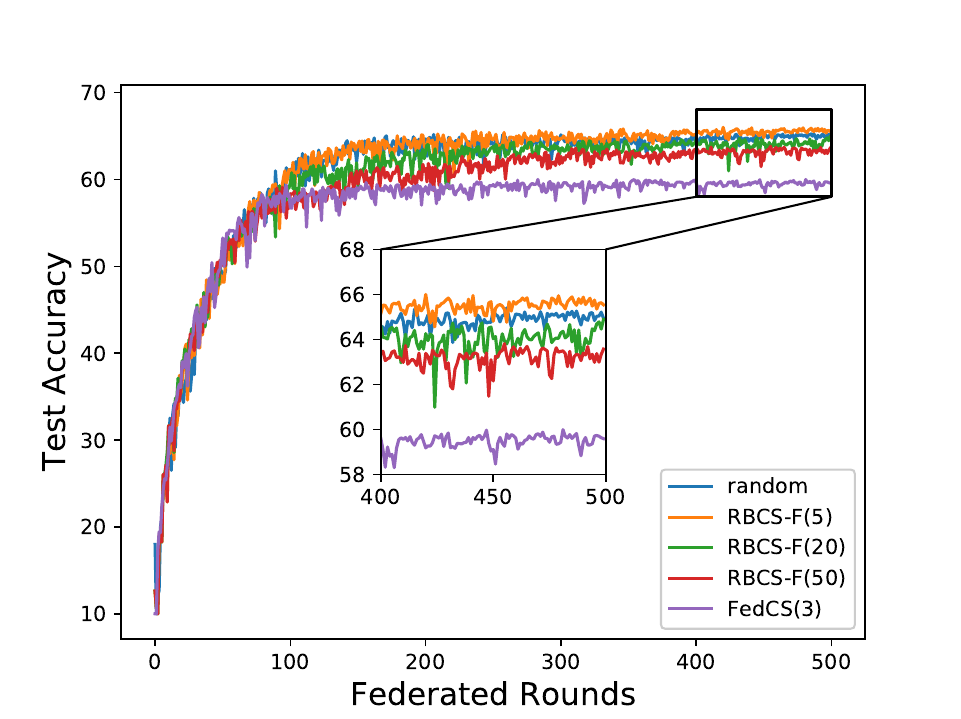}%
}
\subfloat[$\gamma_1=10$]{\includegraphics[width=2.4in]{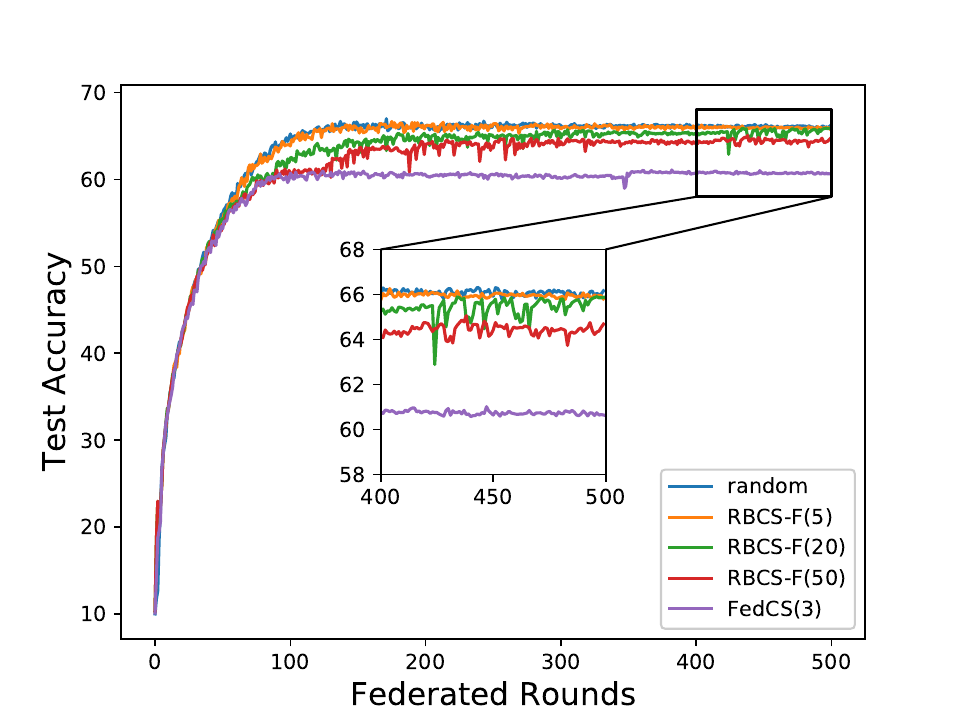}%
}
\subfloat[iid (approximates to $\gamma_1=\infty$ )]{\includegraphics[width=2.4in]{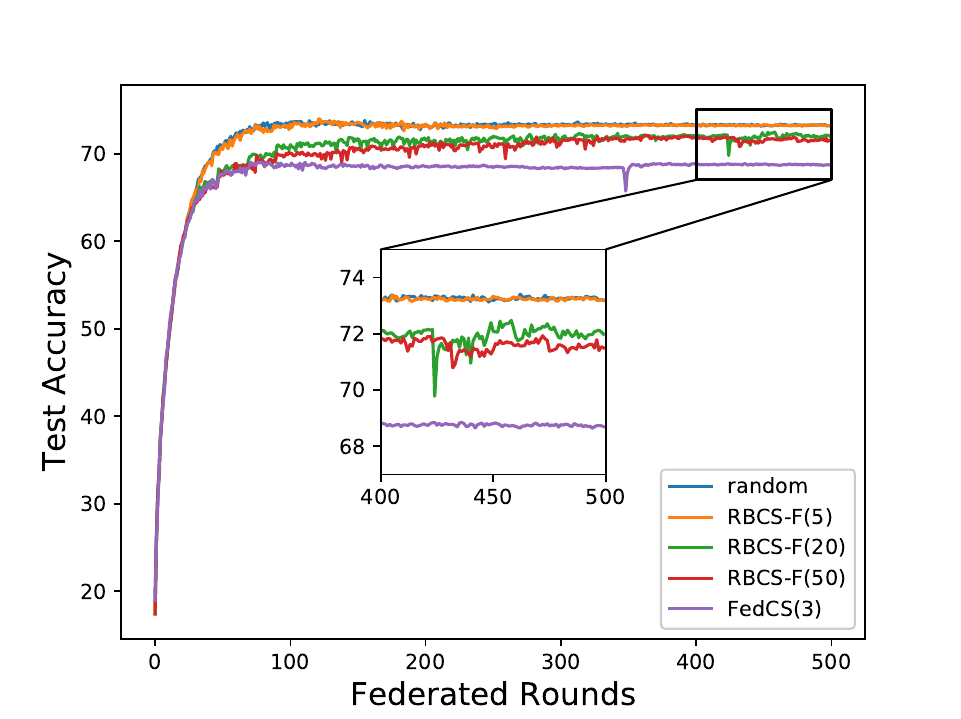}%
}
\caption{Accuracy vs. federated rounds for fashion-MNIST ( $(a), (b), (c)$ ) and CIFAR-10 ( $(d), (e), (f)$ )}
\label{acc rounds fix}
\end{figure*}
Fig. \ref{acc rounds fix} depicts how our proposed RBCS-F with different settings of $V$ performs in the real training, being compared with the baselines, random and FedCS(3). The result is consistent with our former conjecture that RBCS-F with a smaller $V$, which is known to be fairer, would achieve a higher final accuracy after rounds of training. Random, a categorically fair scheme, yields the best performance in terms of final model accuracy, while the FedCS(3), another extreme in terms of fairness, does not promise us a commensurate result. \par
Another point we are interested in is that RBCS-F with a higher penalty seems to spend more rounds to reach a certain accuracy, which we refer to a lower \textit{round efficiency}. This phenomenon can be justified by the result from Fig. \ref{queue status}, which indicates that RBCS-F with a higher penalty tends only to consider fairness when the queue length is large, or in other words, only during a big number of training rounds. Correspondingly, the delay of fairness consideration would make the global model having the chance of aggregating some seldom access data only when the number of rounds is large, and thereby, causing postpone on convergence. Besides, we also found that RBCS-F generally outperforms FedCS(3) in terms of round efficiency, which is conspicuously depicted by subfigures (e) and (f).
\subsubsection{Accuracy vs. training time}
It is also of interest to see how the strategies perform in terms of \textit{time efficiency} (i.e. the elapsed time to reach a certain accuracy). As such, we take advantage of Fig. \ref{Accurarcy vs. training time} to show our investigation on how the model test accuracy evolves according to the elapse of training time. There are a few observations we can derive here. First, a fairer scheme generally achieves a higher final accuracy (or a higher convergence value), which is not a new result as we have already corroborated it in Section \ref{Accuracy vs. federated round}. Second, in terms of time efficiency, FedCS(3) achieves an outstanding performance when accuracy is low, which critically outperforms all other schemes, while random, shown to have the worst performance. It is not surprising to see that FedCS(3) could reach the highest time efficiency in the first few rounds as it has complete access to the client-information and has no regard for the fairness factor, but we also see that our proposed algorithm RBCS-F has an acceptable performance gap compared with FedCS(3) and achieves a significant improvement compared with random. \par
Finally, it is interesting to compare the time efficiency of RBCS-F with different penalties. During our investigation, we see that a higher penalty does not necessarily bring us a higher time efficiency. Specifically, from Figs. \ref{Accurarcy vs. training time}(e) and \ref{Accurarcy vs. training time}(f), we see that RBCS-F(20) outperforms RBCS-F(50) during the whole training process. The phenomenon sounds weird as we can derive from Fig. \ref{time_training} that the average time for each round is strictly decreasing with penalty $V$. But actually, it can be explained if taking into account the result from Section \ref{Accuracy vs. federated round} that the convergence round would also increase with penalty $V$. The tradeoff between convergence round and training time of each round is what we believe to cause the unexpected phenomenon.
\begin{figure*}[!hbtp]
\centering
\subfloat[$\gamma_1=1$]{\includegraphics[width=2.5in]{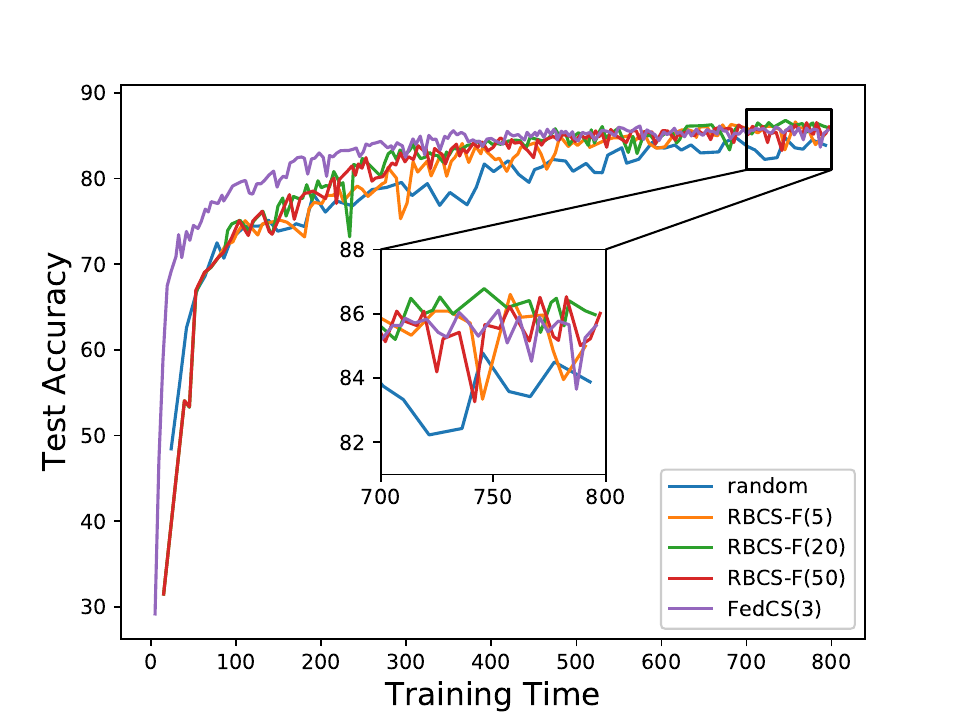}%
}
\subfloat[$\gamma_1=10$]{\includegraphics[width=2.5in]{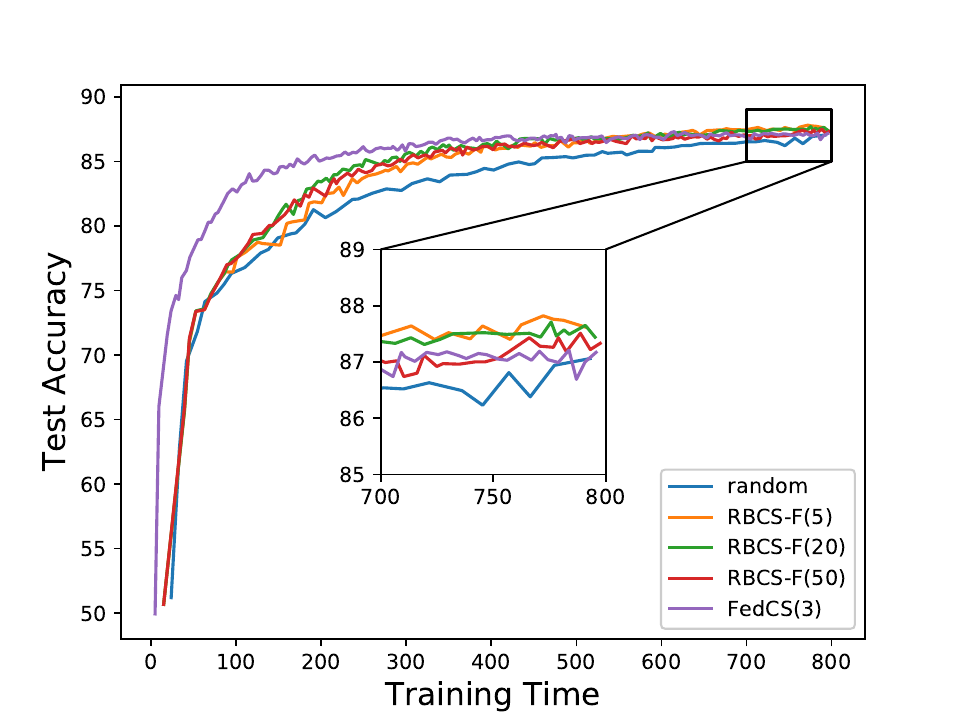}%
}
\subfloat[iid (approximates to $\gamma_1=\infty$ )]{\includegraphics[width=2.5in]{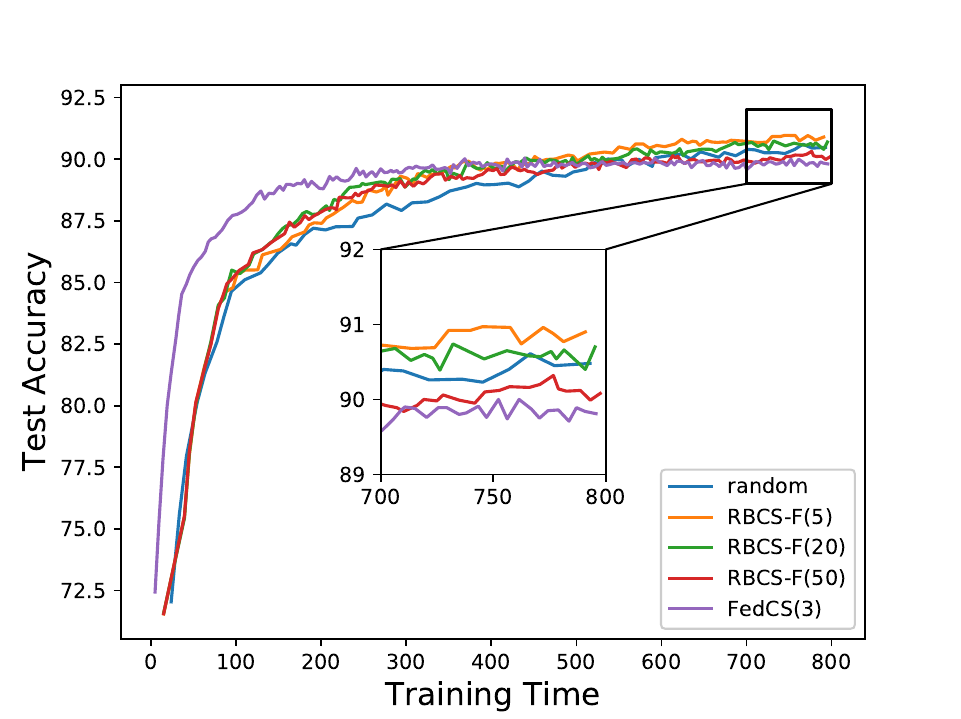}%
}\\
\subfloat[$\gamma_1=1$]{\includegraphics[width=2.5in]{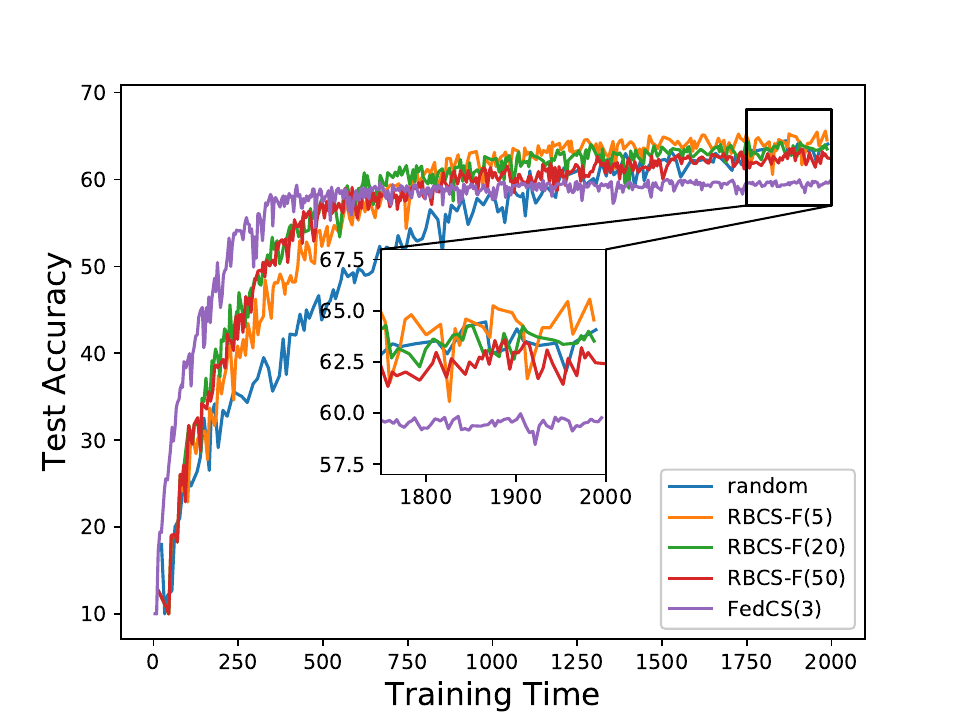}%
}
\subfloat[$\gamma_1=10$]{\includegraphics[width=2.5in]{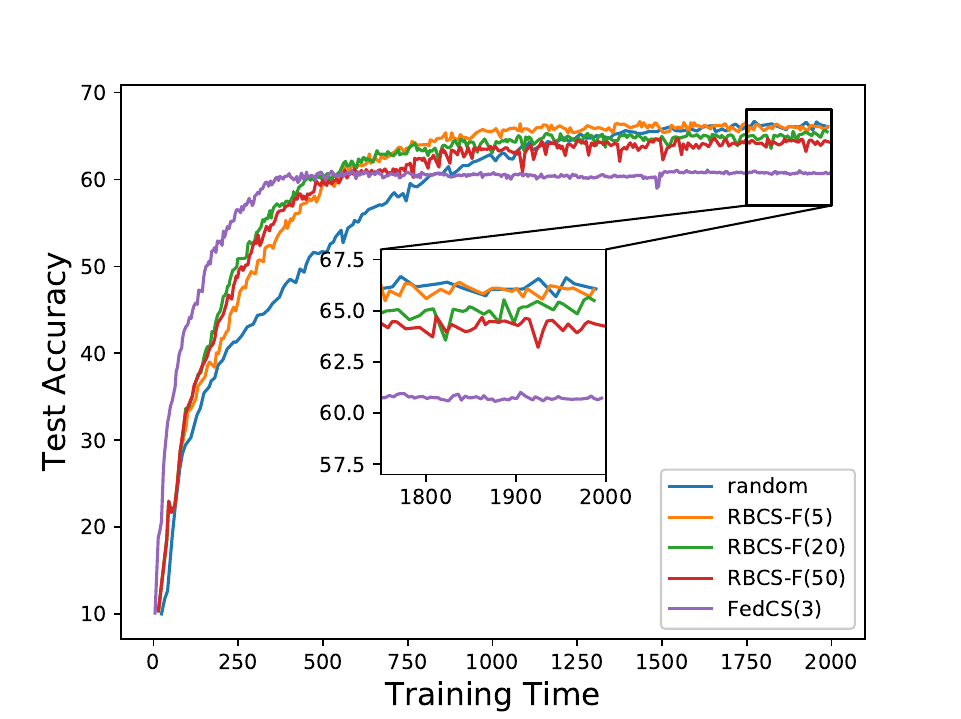}%
}
\subfloat[iid (approximates to $\gamma_1=\infty$ )]{\includegraphics[width=2.5in]{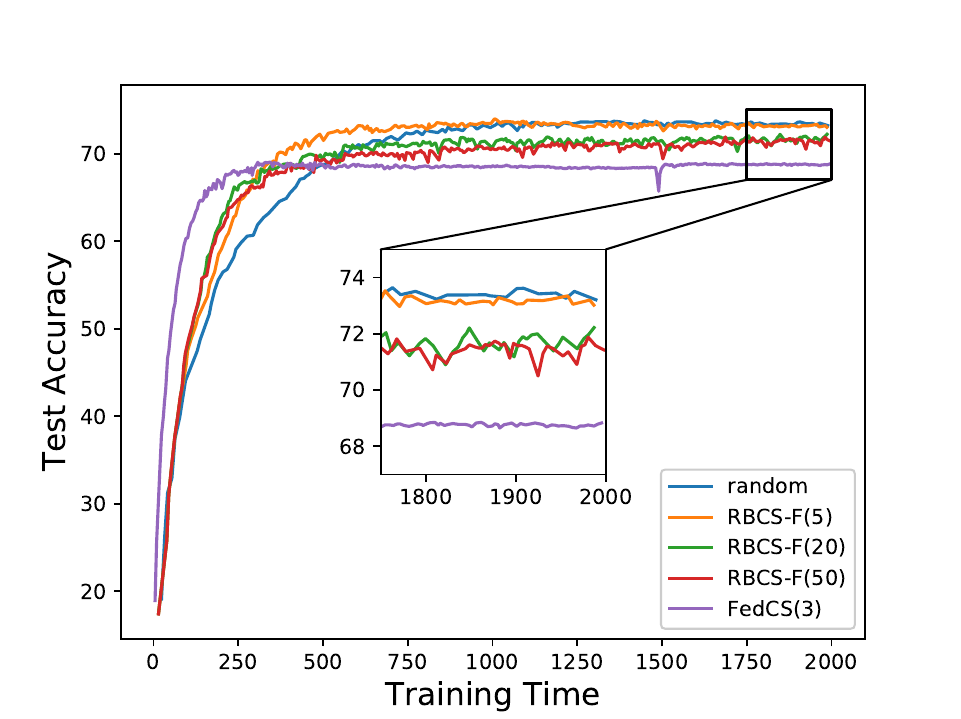}%
}

\caption{Accuracy vs. training time for fashion-MNIST ( $(a), (b), (c)$ ) and CIFAR-10 ( $(d), (e), (f)$ )}
\label{Accurarcy vs. training time}
\end{figure*}
\section{Conclusion and Future Prospect}
In this paper, we have investigated the client selection problem for federated learning. Our concern mainly focuses on the tradeoff between fairness factor and training efficiency. In light of the experiment on our proposed method, we found that fairness is indeed playing a critical role in the training process. In particular, we show that a fairer strategy could promise us a higher final accuracy while inevitably sacrificing a few training efficiency. In terms of how the fairness factor would affect the final achieved accuracy, as well as the convergence speed, however, we could not figure out a rigorous way to quantify their relation. And neither could we track down from the existing literature any theoretical analysis of the fairness factor for FL, making this particular issue quite worthy of investigation. Our future effort would be mainly on this emerging issue.

\small{
\section*{Acknowledgment}
This work is supported by National Natural Science Foundation of China (Grant Nos. 61872084, 61772205), Guangzhou Science and Technology Program key projects (Grant Nos. 202007040002 and 201902010040, Guangdong Major Project of Basic and Applied Basic Research (2019B030302002).
}

\ifCLASSOPTIONcaptionsoff
  \newpage
\fi



%
\bibliographystyle{IEEEtran}
 \tiny
\bibliography{RBCS-F}{}
%

%
%
%



\newpage
\if\arxiv0 \normalsize
\appendices
\section{Proof of Theorem \ref{first upper bound for drift plus cost} }
\label{Proof1}
\begin{proof}
Taking square of  (\ref{Z queue}), we have 
 \begin{align}
\begin{split}
 Z^2_{t+1,n} = Z^2_{t,n}+2Z_{t,n}\left(\beta-x_{t,n} \right)+  \left(\beta-x_{t,n}\right)^2
\end{split}
\end{align}
Then the difference between $\frac{1}{2}Z^2_{t+1,n}$ and $\frac{1}{2}Z^2_{t,n}$ becomes:
\begin{align}
\label{Z difference}
 \begin{split}
&\frac{1}{2}\left(Z_{t+1,n}^2-Z_{t,n}^2\right)\\&= \frac{1}{2}\left(\beta-x_{t,n}\right)^2+Z_{t,n}(\beta-x_{t,n} )  \\ 
&\leq \frac{1}{2}\left( x_{t,n}^2+ \beta^2 \right)+Z_{t,n}\left(\beta-x_{t,n}\right)\\
&\overset{(a)}\leq \frac{1}{2} ( 1+\beta^2 )+Z_{t,n}\left(\beta-x_{t,n}\right)
\end{split}
\end{align}
Among which, (a) is valid since $x_{t,n}\in\{0,1\}$, trivially we have $x_{t,n}^2\leq 1$. 

Now combining (\ref{lyapunov function}), (\ref{drift}) and (\ref{Z difference}), it yields: 
\begin{align}
\label{drift bound}
 \begin{split}
\Delta(\Theta(t)) \leq & \Gamma+ \sum_{n\in\mathcal{N}} Z_{t,n} \mathbb{E}[  \beta-x_{t,n} |  \boldsymbol{\Theta}(t)  ]\\
\end{split}
\end{align}

where  $\Gamma= N\left(1+\beta^2 \right)/2 $.
\par
Plugging $V \mathbb{E}[  f(\mathcal{S}_t,\boldsymbol{\tau}_t) | \boldsymbol{\Theta}(t)) ] $ into (\ref{drift bound}), it can smoothly transform to the form in (\ref{drift plus cost bound}). This completes the proof.
\end{proof}

\section{Proof of Lemma \ref{confidence bound}}
\label{Proof2}
 We first give Lemma \ref{theta's confidence bound} as a fundamental preliminary, and then we will show our justification of Lemma \ref{confidence bound}.
 \begin{lemma} [\cite{abbasi2011improved}, \text { Theorem } 2] \label{theta's confidence bound}  Assume that $\epsilon_t $ is conditionally $R$-sub-Gaussian for $R \geq 0$, $\left\|\bm\theta_{n}^{*}\right\|_{2} \leq$
$S$ and $\left\|\mathbf{c}_{t,n}\right\|_{2} \leq L$ for all $t \geq 1$ and $n \in \mathcal{N}$. Here $S$ and $L$ are both positive finite constants. Define $\mathbf{H}_{T,n}=\mathbf{H}+\sum_{t=1}^{T} x_{t,n} \mathbf{c}_{t,n} \mathbf{c}_{t,n}^{\top}$ and $\operatorname{set} \mathbf{H}=\lambda \mathbf{I}$.
Then, with probability at least $1-\delta,$ for all rounds $t \geq 1$, $\hat{\bm \theta}_{t,n}$ satisfies:
\begin{equation}
\left\|\hat{\bm\theta}_{t,n}-\bm\theta_{n}^{*}\right\|_{\mathbf{H}_{t-1,n}} \leq  R \sqrt{3 \log \left(\frac{1+t L^2 / \lambda}{\delta}\right)}+\lambda^{1 / 2} S
\end{equation}
where  we denote $\|\mathbf{b}\|_{\mathbf{M}} \triangleq \sqrt{\mathbf{b}^{T} \mathbf{M} \mathbf{b}}$. $\mathbf{b}$ is a vector and $\mathbf{M}$ is a
positive definite matrix.
\end{lemma}
The proof is ommited here for brevity. We refer to \cite{abbasi2011improved} for a more dedicated justification. 
\begin{remark}
Recall that $\epsilon_t $ in our formulation is assumed to be an  stochastic variable following a conditionally R-sub-Gaussian, the first assumption trivially holds. In addition, as we have assumed that there exist concrete maximum bounds for $\tau_n^{b}$, $\tau_n^s$ and $1/\eta$,  we can always find a positive $S< \infty$ making $\left\|\bm\theta_{n}^{*}\right\|_{2} \leq S$. 
 Also, we can ensure $\left\|\mathbf{c}_{t,n}\right\|_{2} \leq L$ due to the finite  volume of contexts. The above analysis justifies the applicability of Lemma \ref{theta's confidence bound} to our system model.
\end{remark}
Now we are going to show the proof of Lemma \ref{confidence bound}.
\begin{proof}
\par
From our estimation rule shown in (\ref{UCB estimation}) we can derive:
\begin{equation}
\label{min zero bound}
\begin{aligned}
& {\tau}_{t,n}^*-\bar{\tau}_{t,n}\\
=& \mathbf{c}_{t,n}^{\top}\bm \theta^{*}_{n}-\max
\left\{0,\mathbf{c}_{t,n}^{\top}\hat{\bm \theta}_{t,n}-\alpha_{t}\left\|\mathbf{c}_{t,n}\right\|_{\mathbf{H}_{t-1,n}^{-1}}\right\} \\
=& \min\left\{ \mathbf{c}_{t,n}^{\top}\bm \theta^{*}_{n}, \quad \mathbf{c}_{t,n}^{\top}\bm \theta^{*}_{n}-(\mathbf{c}_{t,n}^{\top}\hat{\bm \theta}_{t,n}-\alpha_{t}\left\|\mathbf{c}_{t,n}\right\|_{\mathbf{H}_{t-1,n}^{-1}})\right\} \\
\end{aligned}
\end{equation}
Now we focus on the second term in the $\min$ function. 
\begin{equation}
\begin{aligned}
&\mathbf{c}_{t,n}^{\top}\bm \theta^{*}_{n}-(\mathbf{c}_{t,n}^{\top}\hat{\bm \theta}_{t,n}-\alpha_{t}\left\|\mathbf{c}_{t,n}\right\|_{\mathbf{H}_{t-1,n}^{-1}}) \\
=&\mathbf{c}_{t,n}^{\top} \left(\bm \theta_{n}^{*}-\hat{\bm \theta}_{t,n}\right)+\alpha_{t}\left\|\mathbf{c}_{t,n}\right\|_{\mathbf{H}_{t-1,n}^{-1}} \\
\geq &-\left\|\bm \theta_{n}^{*}-\hat{\bm \theta}_{t,n}\right\|_{\mathbf{H}_{t-1,n}}\left\|\mathbf{c}_{t,n}\right\|_{\mathbf{H}_{t-1,n}^{-1}}+\alpha_{t}\left\|\mathbf{c}_{t,n}\right\|_{\mathbf{H}_{t-1,n}^{-1}} \\
\overset{(a)}\geq &-\alpha_{t}\left\|\mathbf{c}_{t,n}\right\|_{\mathbf{H}_{t-1,n}^{-1}}+\alpha_{t}\left\|\mathbf{c}_{t,n}\right\|_{\mathbf{H}_{t-1,n}^{-1}}
\\=&0
\end{aligned}
\end{equation}
where (a) can be derived from Lemma \ref{theta's confidence bound}.
Combining the result with the fact $\mathbf{c}_{t,n}^{\top} \bm \theta^{*}_{n}>0$ and plugging it into (\ref{min zero bound}) yields ${\tau}_{t,n}^*-\bar{\tau}_{t,n}\geq 0$.\par
From a different perspective, we have:
\begin{equation}\begin{split} 
&{\tau}^*_{t,n}-\bar{\tau}_{t,n}
\\=& \mathbf{c}_{t,n}^{\top}\bm \theta^{*}_{n}-\max
\left\{0,\mathbf{c}_{t,n}^{\top}\hat{\bm \theta}_{t,n}-\alpha_{t}\left\|\mathbf{c}_{t,n}\right\|_{\mathbf{H}_{t-1,n}^{-1}}\right\} 
\\\leq& \mathbf{c}_{t,n}^{\top}\bm \theta_{n}^*-
\mathbf{c}_{t,n}^{\top}\hat{\bm \theta}_{t,n}+\alpha_{t}\left\|\mathbf{c}_{t,n}\right\|_{\mathbf{H}_{t-1,n}^{-1}}
\\ \leq &\left|\mathbf{c}_{t,n}^{\top} \left(\bm \theta_{n}^{*}-\hat{\bm \theta}_{t,n}\right) \right|+\alpha_{t}\left\|\mathbf{c}_{t,n}\right\|_{\mathbf{H}_{t-1,n}^{-1}} 
\\ \overset{(a)}\leq &\left\|\bm \theta_{n}^{*}-\hat{\bm \theta}_{t,n}\right\|_{\mathbf{H}_{t-1,n}}\left\|\mathbf{c}_{t,n}\right\|_{\mathbf{H}_{t-1,n}^{-1}}+\alpha_{t}\left\|\mathbf{c}_{t,n}\right\|_{\mathbf{H}_{t-1,n}^{-1}} \\ \leq & 2 \alpha_{t}\left\|\mathbf{c}_{t,n}\right\|_{\mathbf{H}_{t-1,n}^{-1}} \end{split}\end{equation}
where (a) is valid by Cauchy-Schwarz inequality. The above results complete the proof.
\end{proof}
\section{Proof of Theorem \ref{regret bound}}
\label{Proof3}
\begin{proof}
The definition of the time average regret gives:
\begin{equation}
\begin{split} R(T) &=\frac{1}{T}\sum_{t=1}^{T} \mathbb{E}\left[  {f}(\mathcal{S}_t,\bm \tau_t)- {f}(\mathcal{S}^*_t,\bm\tau_t)\right]
\\&=\frac{1}{T}\sum_{t=1}^{T}\mathbb{E}\left[ \max_{n \in \mathcal{S}_t}\{ \tau_{t,n}   \}-\max_{n \in \mathcal{S}^*_t}\{ \tau_{t,n}   \}    \right]
\end{split}
\end{equation}
Let $\Delta R(t) =\mathbb{E}\left[ \max_{n \in \mathcal{S}_t}\{ \tau_{t,n}   \}-\max_{n \in \mathcal{S}^*_t}\{ \tau_{t,n}   \}    \right] $ capture the estimated rewards gap between optimal policy and the policy obtained by RBCS-F. 
\par 
Taking expectation of (\ref{drift bound}) with respect to $ \boldsymbol{\Theta}(t)  $, it yields:
\begin{align}
 \begin{split}
 \label{drift bound without expectation}
\mathbb{E}\left[ \mathcal {L}(\boldsymbol{\Theta}(t+1))-\mathcal {L}(\boldsymbol{\Theta}(t))\right] \leq & \Gamma+  \sum_{n\in \mathcal{N}} \mathbb{E}[ Z_{t,n}  ( \beta-x_{t,n}) ]
\end{split}
\end{align}
Combining the definition of $\Delta R(t)$ and (\ref{drift bound without expectation}), we have
\begin{align}
\label{regret bound 0}
 \begin{split}
&\mathbb{E}\left[ \mathcal {L}(\boldsymbol{\Theta}(t+1))-\mathcal {L}(\boldsymbol{\Theta}(t))\right]+V\Delta R(t) 
\\&\leq \Gamma+ \sum_{n\in \mathcal{N}}\mathbb{E}[ Z_{t,n}  ( \beta-x_{t,n}) ]+V\mathbb{E}\left[ \max_{n \in \mathcal{S}_t}\{ \tau_{t,n}   \}-\max_{n \in \mathcal{S}^*_t}\{  \tau_{t,n}    \}\right]
\\&= \Gamma+ \mathbb{E}\left[ V \max_{n \in \mathcal{S}_t}\{ \tau_{t,n}   \}-\sum_{n\in \mathcal{S}_t} Z_{t,n}\right] 
\\&\quad- \mathbb{E}\left[V \max_{n \in \mathcal{S}^*_t}\{ \tau_{t,n}   \}-\sum_{ n \in \mathcal{S}_t^*}Z_{t,n}\right]+\sum_{n\in \mathcal{N}}\mathbb{E}\left[ Z_{t,n} (  \beta-x_{t,n}^*)\right]
\\&\overset{(a)}\leq \Gamma+ \mathbb{E}\left[ V \max_{n \in \mathcal{S}_t}\{ \tau_{t,n}   \}-\sum_{n\in \mathcal{S}_t} Z_{t,n}\right] 
\\& \quad - \mathbb{E}\left[V \max_{n \in \mathcal{S}_t^{*}}\{ \tau_{t,n}   \}-\sum_{n\in \mathcal{S}_t^*}Z_{t,n}\right]
\end{split}
\end{align}
where (a) is valid since $\mathbb{E}\left[ Z_{t,n}(\beta-x_{t,n}^*)\right] \leq0$. This is trivially true because the optimal policy has to ensure the mean rate stability of the fairness queue, which implies that its expected input rate must be smaller than its service rate.
\par
Summing (\ref{regret bound 0}) over $t \in \{1,2,\dots,T\}$ for some $T>0$ using telescope yields:
 \begin{align} 
 \label{regret bound 1}
 \begin{split}
&\mathbb{E}[\mathcal{L}(\boldsymbol{\Theta}(T))]-\mathbb{E}[\mathcal{L}(\boldsymbol{\Theta}(0))]+ V\sum_{t=1}^{T }\Delta R(t) \\& \leq T\Gamma+\sum_{t=1}^{T}\mathbb{E}\left[ G_1(t) \right]
 \end{split}
\end{align}
where 
 \begin{align} 
 \begin{split}
 G_1(t)=&\left(V \max_{n \in \mathcal{S}_t}\{ \tau_{t,n}   \}-\sum_{n\in \mathcal{S}_t} Z_{t,n}\right)\\& -\left(V \max_{n \in \mathcal{S}_t^{*}}\{ \tau_{t,n}   \}-\sum_{n\in \mathcal{S}_t^*}Z_{t,n}\right)
 \end{split}
 \end{align} 
Recall that $\mathbb{E}[L(\boldsymbol{\Theta}(T))] \geq 0$ according to its definition, then we can reformulate (\ref{regret bound 1}) into the following form:
 \begin{align}
 \label{bound 2}
\frac{1}{T} \sum_{t=1}^T\Delta R(t) \leq \frac{\Gamma}{V}+\sum_{t=1}^T\frac{\mathbb{E}\left[ G_1(t)\right]}{TV}+\frac{\mathbb{E}[L(\boldsymbol{\Theta}(0))]}{TV}
\end{align}
Since $\mathbb{E}[L(\boldsymbol{\Theta}(0))]=0$, we have 
 \begin{align}
 \label{bound 3}
R(T)=\frac{1}{T} \sum_{t=1}^T\Delta R(t) \leq  \frac{\Gamma}{V}+\sum_{t=1}^T\frac{\mathbb{E}\left[ G_1(t)\right]}{TV}
\end{align}
Then we proceed to bound $G_1(t)$.
Assume a policy $\pi^{\prime}$ makes her decision $\mathcal{S}^{\prime}_t$ following such a manner:
 \begin{align}
 \label{pi objective}
\mathcal{S}^{\prime}_t=\underset{\mathcal{S}^{\prime}_t \in \mathcal{C}_t}{\operatorname{argmin}} \quad V \max_{n \in \mathcal{S}^{\prime}_t}\{ \tau_{t,n}   \}-\sum_{n\in \mathcal{S}^{\prime}_t}Z_{t,n}
\end{align}
Here $\mathcal{C}_t$ captures all the possible solutions confined by constraints (explicitly, constraints in \textit{P4}). The only difference between policy $\pi^{\prime}$ and our proposed one is that $\pi^{\prime}$ can be fully aware of the real model exchange time. But recall that our algorithm has minimized the objective function in \textit{P4}, trivially we know that:
 \begin{align}
 \label{haha}
 \begin{split}
V \max_{n \in \mathcal{S}_t}\{ \bar{\tau}_{t,n}  \}-\sum_{n\in \mathcal{S}_t}Z_{t,n} \leq  V \max_{n \in \mathcal{S}^{\prime}_t}\{ \bar{\tau}_{t,n}  \}-\sum_{n\in \mathcal{S}^{\prime}_t}Z_{t,n}
 \end{split}
\end{align}
Likewise, not a single policy, including the optimal policy $\pi^*$,  can outperform $\pi^{\prime}$ in her objective as shown in (\ref{pi objective}), trivially we can derive:
 \begin{align}
  \label{haha1}
  \begin{split}
V \max_{n \in \mathcal{S}^{\prime}_t}\{ \tau_{t,n}   \}-\sum_{n\in \mathcal{S}^{\prime}_t}Z_{t,n}\leq  V \max_{n \in \mathcal{S}^{*}_t}\{ \tau_{t,n}   \}-\sum_{n\in \mathcal{S}^{*}_t}Z_{t,n}
 \end{split}
\end{align}
Therefore, according to (\ref{haha}) and (\ref{haha1}), it follows that:
 {\small
 \begin{align}
 \begin{split}
&G_1(t)\\=&  \left(V \max_{n \in \mathcal{S}_t}\{ \tau_{t,n}   \}-\sum_{n\in \mathcal{S}_t} Z_{t,n}\right) -\left(V \max_{n \in \mathcal{S}_t^{*}}\{ \tau_{t,n}   \}-\sum_{n\in \mathcal{S}_t^*}Z_{t,n}\right)
\\ \leq&    \left(V \max_{n \in \mathcal{S}_t}\{ \tau_{t,n}   \}-\sum_{n\in \mathcal{S}_t} Z_{t,n}\right) -\left(V \max_{n \in \mathcal{S}_t^{\prime}}\{ \tau_{t,n}   \}-\sum_{n\in \mathcal{S}_t^{\prime}}Z_{t,n}\right)
\\ \leq&    \left(V \max_{n \in \mathcal{S}_t}\{ \tau_{t,n}   \}-\sum_{n\in \mathcal{S}_t} Z_{t,n}\right) -\left(V \max_{n \in \mathcal{S}_t^{\prime}}\{ \tau_{t,n}   \}-\sum_{n\in \mathcal{S}_t^{\prime}}Z_{t,n}\right)
\\ +&    \left(V \max_{n \in \mathcal{S}_t^{\prime}}\{ \bar{\tau}_{t,n}  \}-\sum_{n\in \mathcal{S}_t^{\prime}} Z_{t,n}\right) -\left(V \max_{n \in \mathcal{S}_t}\{\bar{ \tau}_{t,n}  \}-\sum_{n\in \mathcal{S}_t}Z_{t,n}\right)
\\ \leq& V \left[ \left(\max_{n \in \mathcal{S}_t}\{ \tau_{t,n}  \}- \max_{n \in \mathcal{S}_t} \{ \bar{\tau}_{t,n} \}\right) +  \left(\max_{n \in \mathcal{S}_t^{\prime}} \{ \bar{\tau}_{t,n} \}-\max_{n \in \mathcal{S}_t^{\prime}}\{ \tau_{t,n}  \}\right)\right]
\\ \leq& V \left[ \left( \tau_{t,s} -  \bar{\tau}_{t,s} \right)+ \left( \bar{\tau}_{t,s^{\prime}}- \tau_{t,s^{\prime} }\right)\right]
 \end{split}
\end{align}
}
where arm $s$ is the arm that get the maximum $\tau_{t,n}$ among the set $n \in \mathcal{S}_t$. Likewise, $s^{\prime}$ is another arm that gets the maximum $\bar{\tau}_{t,n}$ among the set $n \in \mathcal{S}_t^{\prime}$.
Notice that both $\bar{\tau}_{t,s}$ and $\bar{\tau}_{t,s^{\prime}}$ are not coupled with the stochastic value $\tau_{t,s}$ and we also have $\mathbb{E}[\tau_{t,s} ]=\tau_{t,s}^*$, $\mathbb{E}[\tau_{t,s^{\prime}} ]= \tau_{t,s^{\prime}}^*$. Subsequently, we have:
 \begin{align}
 \begin{split}
\mathbb{E}[G_1(t)]  \leq& V \left[ \underbrace{ \left( \tau_{t,s}^* -  \bar{\tau}_{t,s} \right) }_{J_1(t)}+ \underbrace{ \left( \bar{\tau}_{t,s^{\prime}}- \tau_{t,s^{\prime} }^*\right)}_{J_2(t)}\right]
 \end{split}
\end{align}
Combining  Lemma \ref{confidence bound},  we can conclude that with probability at least $(1-\delta)^2$, $J_1(t) \leq 2 \alpha_{t}\left\|\mathbf{c}_{t,s}\right\|_{\mathbf{H}_{t-1,s}^{-1}}$ and $ J_2(t)\leq0$. Now we know that $\mathbb{E}[G_1(t)]  \leq V J_1(t) $ with large probability.
\par Note that we have a concret bound for $\tau_{t,s}^{*}$ since elements of $\bm c_{t,n}$ and $\boldsymbol{\theta}_{n}^*$ are all bounded. Now assume that we have $\tau_{t,s}^{*} \leq K $, then combining $\bar{\tau}_{t,s}\geq 0$, it follows that $J_1(t)\leq K$. Consequently, we have: 
\begin{equation}
\begin{split}
J_1(t)& \leq \min \left\{  2 \alpha_{t}\left\|\mathbf{c}_{t,s}\right\|_{\mathbf{H}_{t-1,s}^{-1}},K  \right\}
\\&\leq  \max\{K,1\}\min \left\{ 2 \alpha_{t}\left\|\mathbf{c}_{t,s}\right\|_{\mathbf{H}_{t-1,s}^{-1}},1  \right\} 
\\&\leq  \max\{K,1\} \cdot \max\{2 \alpha_{t},1\} \cdot \min \left\{ \left\|\mathbf{c}_{t,s}\right\|_{\mathbf{H}_{t-1,s}^{-1}},1  \right\} 
\end{split}
\end{equation}
where the last two inequalities hold by $\min \{a,b\} \leq \max\{b,1\} \cdot \min\{a,1\} $ and $\min \{ab,1\} \leq \max\{a,1\} \cdot \min\{b,1\} $, respectively. Let $  \zeta_t= \max\{K,1\} \cdot \max\{2 \alpha_{t},1\}$. Now $\mathbb{E}[G_1(t)]$ has been successfully bounded into a closed form: $\mathbb{E}[G_1(t)]\leq V\zeta_t \min \left\{ \left\|\mathbf{c}_{t,s}\right\|_{\mathbf{H}_{t-1,s}^{-1}},1  \right\} $. Then we  continue our proof by given $\Upsilon(t)=\frac{1}{T}\sum_{t=1}^T \zeta_t \min \left\{ \left\|\mathbf{c}_{t,s}\right\|_{\mathbf{H}_{t-1,s}^{-1}},1  \right\} $, which is the upper bound on the second term of  (\ref{bound 3}) that we have derived so far. Before our further bounding on $\Upsilon(t)$, we first familiarize the readers with Lemma \ref{lemma 2}.
\begin{lemma} 
\label{lemma 2}Assume $\mathbf{H}_{T,n}=\mathbf{H}+\sum_{t=1}^{T} x_{t,n} \mathbf{c}_{t,n}\mathbf{c}_{t,n}^{\top}$ and $\mathbf{H}=\lambda \mathbf{I}$.
Then, if $\lambda \geq 1$ and $\left\|\mathbf{c}_{t,n}\right\|_{2} \leq L$ for all  $t$ and $n$, we have
\begin{equation}
\begin{split}
&\sum_{t=1}^{T} \min \left\{\left\|\mathbf{c}_{t,n}\right\|_{H_{t-1,n}^{-1}}^{2},1\right\} \\\leq& 2\left(\log \operatorname{det}\left(\mathbf{H}_{T,n}\right)-\log \operatorname{det}( \mathbf{H})\right) 
\\\leq &6\log \left(\left(\operatorname{trace}(\mathbf{H})+T L^{2}\right) / 3\right)-2\log \operatorname{det} (\mathbf{H})
 \end{split}
\end{equation}
 for any $n$ and $T$.
\end{lemma}
The proof is omitted here for briefness. For more detail, we refer the interested readers to Lemma 11 in \cite{abbasi2011improved}.

With this lemma introduced,  now we shall introduce our further deduction, as presented in the following:
\begin{equation}
\begin{split}
\label{near final}
&\Upsilon(t)= \frac{1}{T}\sum_{t=1}^{T}  \zeta_t \min \left\{ \left\|\mathbf{c}_{t,s}\right\|_{\mathbf{H}_{t-1,s}^{-1}},1  \right\} 
\\&\overset{(a)}\leq \sqrt{ \frac{\sum_{t=1}^{T}\zeta_t^2 \min \left\{ \left\|\mathbf{c}_{t,s}\right\|_{\mathbf{H}_{t-1,s}^{-1}}^2,1  \right\} }{T} }
\\&\leq \zeta_T \sqrt{ \frac{\sum_{t=1}^{T}\min \left\{ \left\|\mathbf{c}_{t,s}\right\|_{\mathbf{H}_{t-1,s}^{-1}}^2,1  \right\}}{T} }
\\&\overset{(b)}\leq \zeta_T \sqrt{ \frac{6 \log ((\operatorname{trace}(\mathbf{H})+TL^2) / 3)-2 \log \operatorname{det}( \mathbf{H})}{T} }
\\&\overset{(c)}=\zeta_T \sqrt{ \frac{6 \log ( \lambda+TL^2/ 3)-6 \log  \lambda}{T} }
\\&= \zeta_T \sqrt{ \frac{6 \log ( 1+T L^2/ 3 \lambda)}{T} }
\end{split}
\end{equation}
where (a) can be justified by arithmetic means  inequality and Theorem \ref{lemma 2} gives (b).  Using the facts $\operatorname{trace}(\mathbf{H})=3\lambda$ and $\operatorname{det} (\mathbf{H})=\lambda^3$ we have (c).
\par
Plugging the above result into (\ref{bound 2}) and combining the fact that $\Gamma=N\left(1+\beta^2 \right)/2$, finally we reach a high probability (i.e. with probability $(1-\delta)^2$) upper bound on the regret, as presented in the following:
\begin{equation}
\begin{split}
R(T)\leq&  \frac{N\left(1+\beta^2 \right)}{2V}+ \zeta_T \sqrt{ \frac{6 \log ( 1+T L^2/ 3 \lambda)}{T} }
\end{split}
\end{equation}
where\par
 $\zeta_T=\max\{K,1\} \cdot \max\left\{2 R \sqrt{3 \log \left(\frac{1+T  L^2/ \lambda}{\delta}\right)}+\lambda^{1 / 2} S,1 \right\}$
 This completes the proof.
\end{proof}

\section{Proof of Theorem \ref{mean rate stable}}
\label{Proof4}
Here we first prepare the readers with Lemma \ref{w-only}, which we will use in our proof of Theorem \ref{mean rate stable}.
\par
\begin{lemma}
\label{w-only}
 For any $\delta^{\prime}> 0$, there exists an $\omega$-only policy $\pi$, which makes independent, stationary, and randomized decisions in every round $t$ based only on the observed stochastic events $\omega(t)$, gives the following results:
\begin{equation}
\begin{aligned} \mathbb{E}\left[ V \max_{n \in \mathcal{N}}\{ x_{t,n}^{\pi} \bar{\tau}_{t,n}  \}| \boldsymbol{\Theta}(t)\right] &=\mathbb{E}\left[ V \max_{n \in \mathcal{N}}\{ x_{t,n}^{\pi} \bar{\tau}_{t,n}  \}\right] 
\\&\leq V \max_{n \in \mathcal{N}}\{ x_{t,n}^{opt} \bar{\tau}_{t,n}  \}+\delta^{\prime} \\  
\mathbb{E} \left \{\beta-x_{t,n}^{\pi}| \boldsymbol{\Theta}(t)\right\} &=\mathbb{E}\left\{\beta-x_{t,n}^{\pi}\right\} \leq \delta^{\prime}  \quad \forall n \in \mathcal{N}
 \end{aligned}
\end{equation}
where the expectations here are with respect to random actions achieved by the policy and the stochastic event $\omega(t)$. $x_{t,n}^{opt}$ is the optimal policy that minimizes over $\max_{n \in \mathcal{N}}\{ x_{t,n} \bar{\tau}_{t,n}  \}$ while meeting the long-term fairness constraint.
\end{lemma}
\par
 The proof of Lemma \ref{w-only} is omitted here, we refer the readers to Theorem 4.5 in \cite{neely2010stochastic}.
\par
Formally, now we begin our justification of Theorem \ref{mean rate stable}.
\begin{proof}
Since RBCS-F minimizes the R.H.S of inequation (\ref{drift plus cost bound}), straightforwardly not other policies could match up with RBCS-F in its objective, namely, it gives:
 \begin{align}
  \begin{split}
\sum_{n\in \mathcal{N}} Z_{t,n} (\beta-x_{t,n}) &+V \max_{n \in \mathcal{N}}\{ x_{t,n} \bar{\tau}_{t,n}  \} \\& \leq \sum_{n\in \mathcal{N}} Z_{t,n} (\beta-x_{t,n}^{\pi}) +V \max_{n \in \mathcal{N}}\{ x_{t,n}^{\pi} \bar{\tau}_{t,n}  \}
 \end{split}
 \end{align}
The above inequality is valid for any $\omega$-only policy $\pi$ and under any $\omega(t)$ and $\boldsymbol{\Theta(t)}$.
Therefore we have:
 \begin{align}
 \begin{split}
&\mathbb{E}\left[\sum_{n\in \mathcal{N}} Z_{t,n} (\beta-x_{t,n}) +V \max_{n \in \mathcal{N}}\{ x_{t,n} \bar{\tau}_{t,n}  \} |\boldsymbol{\Theta}(t)\right] \\& \leq \mathbb{E}\left[\sum_{n\in \mathcal{N}} Z_{t,n} (\beta-x_{t,n}^{\pi}) +V \max_{n \in \mathcal{N}}\{ x_{t,n}^{\pi} \bar{\tau}_{t,n}  \}|\boldsymbol{\Theta}(t) \right]
 \end{split}
 \end{align}
Plugging this into (\ref{drift plus cost bound}), it yields:

 \begin{align}
 \begin{split}
 &\Delta(\Theta(t))+V \mathbb{E}[\max_{n \in \mathcal{N}}\{ x_{t,n} \bar{\tau}_{t,n}  \}| \boldsymbol{\Theta}(t)]
 \\& \leq \Gamma+\sum_{n\in \mathcal{N}} Z_{t,n} \mathbb{E}[\beta-x_{t,n}^{\pi}|\boldsymbol{\Theta}(t)] \\&+V\mathbb{E}[\max_{n \in \mathcal{N}}\{ x_{t,n}^{\pi} \bar{\tau}_{t,n}\}  |\boldsymbol{\Theta}(t)]
\\& \overset{(a)}\leq \Gamma+ \sum_{n\in \mathcal{N}}Z_{t,n} \delta^{\prime} + V \max_{n \in \mathcal{N}}\{ x_{t,n}^{opt} \bar{\tau}_{t,n}  \}+\delta^{\prime}
 \end{split}
\end{align}
where inequality (a) could be derived from Lemma \ref{w-only}. Then taking $\delta^{\prime} \to 0$  yields:
 \begin{align}
  \begin{split}
\Delta(\Theta(t))+V \mathbb{E}[\max_{n \in \mathcal{N}}\{ x_{t,n} \bar{\tau}_{t,n}  \}| \boldsymbol{\Theta}(t)] \\ \leq\Gamma+ &V \max_{n \in \mathcal{N}}\{ x_{t,n}^{opt} \bar{\tau}_{t,n}  \}
\end{split}
\end{align}
Taking expectation on the above ineqality with repect to $\boldsymbol{\Theta}(t)$ and then summing them over $t \in \{1,2,\dots,T\}$ for $T\to \infty$ using telescope yields:
 \begin{align}
 \label{kkkk}
 \begin{split}
&\mathbb{E}[L(\boldsymbol{\Theta}(T))]-\mathbb{E}[L(\boldsymbol{\Theta}(0))]+V\sum_{t=1}^T \mathbb{E}[\max_{n \in \mathcal{N}}\{ x_{t,n} \bar{\tau}_{t,n}  \}] \\&\leq T\Gamma+ V \sum_{t=1}^T \max_{n \in \mathcal{N}}\{ x_{t,n}^{opt} \bar{\tau}_{t,n}  \}
 \end{split}
\end{align}
In light of the law of large numbers, we have 
\begin{align}
\sum_{t=1}^T \max_{n \in \mathcal{N}}\{ x_{t,n}^{opt} \bar{\tau}_{t,n}  \}=\sum_{t=1}^T    \mathbb{E}[ \max_{n \in \mathcal{N}}\{ x_{t,n}^{opt} \bar{\tau}_{t,n}  \}]
\end{align}
In addtion, recall that not a policy can obtain smaller objective than the optimal scheme, therefore, it yields:
\begin{align}
 \mathbb{E}[ \max_{n \in \mathcal{N}}\{ x_{t,n}^{opt} \bar{\tau}_{t,n}\}] \leq \mathbb{E}[\max_{n \in \mathcal{N}}\{ x_{t,n} \bar{\tau}_{t,n}  \}] \quad \forall t \in \mathcal{T}
\end{align}
 Reranging (\ref{kkkk}), it gives:

 \begin{align}
\label{aeeeaa}
 \mathbb{E}[L(\boldsymbol{\Theta}(T))]-\mathbb{E}[L(\boldsymbol{\Theta}(0))] \leq  T\Gamma
 \end{align}
Plugging the definition of $\mathcal{L}(\boldsymbol{\Theta}(T))$ in (\ref{aeeeaa}) yields:

 \begin{align}
\sum_{n \in \mathcal{N}}\mathbb{E}[Z_{T,n}^2)] \leq 2T\Gamma+2\mathbb{E}[L(\boldsymbol{\Theta}(0)]
\end{align}
 Plugging the fact $\mathbb{E} [Z_{T,n}]^2 \leq \mathbb{E}[Z_{T,n}^2] $ into the above inequality, it yields:
 \begin{align}
\sum_{n \in \mathcal{N}}\mathbb{E}\left [Z_{T,n}\right] \leq \sqrt{2T\Gamma+2\mathbb{E}[L(\boldsymbol{\Theta}(0))]}
\end{align}
Dividing by $T$ yields :
 \begin{align}
\lim _{T \rightarrow \infty} \sum_{n \in \mathcal{N}}\frac{\mathbb{E}\left [Z_{T,n}\right] }{T}\leq \lim _{T \rightarrow \infty} \sqrt{\frac{2\Gamma}{T}+\frac{2\mathbb{E}[L(\boldsymbol{\Theta}(0))]}{T^2}}=0
\end{align}
Since $\mathbb{E}\left [Z_{T,n}\right] \geq 0$ for any $T$ and $n$, finally we conclude that 
\begin{align}
\lim _{T \rightarrow \infty} \frac{\mathbb{E}\left [Z_{T,n}\right] }{T}=0 \quad \forall n \in \mathcal{N}
\end{align}
which explicitly marks the mean rate stability of all queues regardless of the real setting of $V$.
Combining the result from Theorem \ref{fairness queue theorem}, we can ensure no violation on the long-term average fairness constraint.
\end{proof}
\end{document}